\documentclass{article}

 \usepackage[main, final]{neurips_2026}
\usepackage[utf8]{inputenc} % allow utf-8 input
\usepackage[T1]{fontenc}    % use 8-bit T1 fonts
\usepackage{hyperref}       % hyperlinks
\usepackage{url}            % simple URL typesetting
\usepackage{booktabs}       % professional-quality tables
\usepackage{amsfonts}       % blackboard math symbols
\usepackage{nicefrac}       % compact symbols for 1/2, etc.
\usepackage{microtype}      % microtypography
\usepackage{xcolor}         % colors

\usepackage{amsmath}
\usepackage{amssymb}
\usepackage{mathtools}
\usepackage{amsthm}

\usepackage[capitalize,noabbrev]{cleveref}

\theoremstyle{plain}
\newtheorem{theorem}{Theorem}[section]

\newtheorem{lemma}[theorem]{Lemma}
\newtheorem{corollary}[theorem]{Corollary}
\theoremstyle{definition}

\theoremstyle{remark}

\usepackage{hyperref}
\usepackage{url}
\usepackage{amssymb}
\usepackage{amsmath}
\usepackage{amsthm}
\usepackage{enumitem}
\usepackage{graphicx}
\usepackage{color}
\usepackage{algorithm}
\usepackage{algorithmicx}%
\usepackage{algpseudocode}%
\usepackage{mathtools}
\usepackage{cleveref}
\usepackage{rotating} % in preamble
\usepackage{booktabs}
\usepackage{subcaption}
\usepackage{xcolor}
\usepackage{colortbl}

\DeclareMathOperator*{\argmin}{arg\,min}
\DeclareMathOperator*{\argmax}{arg\,max}

\newcommand\defin{\stackrel{\mathclap{\scriptsize\mbox{def}}}{=}}
\newcommand\fit{\stackrel{\mathclap{\scriptsize\mbox{fit}}}{\leftarrow}}

\makeatletter
\newcommand{\markthis}[3]{% #1 = marker, #2 = label, #3 = relation
  \overset{% the marker and the label
    \textup{\makebox[0pt]{#1}}%
    \def\@currentlabel{#1}%
    \ltx@label{#2}%
  }{% the relation
    #3%
  }%
}

\title{Ensemble Distributionally Robust \\ Bayesian Optimisation}

\author{%
  Tigran Ramazyan \\
  Faculty of Computer Science\\
  HSE University\\
  Moscow, Russia \\
  \texttt{tramazyan@hse.ru} \\
  \And
  Denis Derkach \\
  Faculty of Computer Science\\
  HSE University \\
  Moscow, Russia \\
  \texttt{dderkach@hse.ru} \\
}

\begin{document}

\maketitle

\begin{abstract}
  We study zeroth-order optimisation under context distributional uncertainty, a setting commonly tackled using Bayesian optimisation (BO). A prevailing strategy to make BO more robust to the complex and noisy nature of data is to employ an ensemble as the surrogate model, thereby mitigating the weaknesses of any single model. In this study, we propose a novel algorithm for Ensemble Distributionally Robust Bayesian Optimisation that remains computationally tractable while managing continuous context. We obtain theoretical sublinear regret bounds, improving current state-of-the-art results. We show that our method’s empirical behaviour aligns with its theoretical guarantees.
\end{abstract}

\section{Introduction}

Design optimisation is fundamental to engineering and scientific discovery, aiming to identify system configurations that maximise operational efficiency or minimise costs \citep{DORIGO2023100085, ship_opt}. The underlying objective function is usually unknown, expensive to evaluate, and analytically intractable. These systems are treated as black boxes - complex simulators or real-world experiments that return only point-wise values for queried inputs, offering no direct gradient information to guide the optimisation \cite{lhcb_ecal}.

While black-box optimisation spans both gradient-based and gradient-free paradigms, the former is largely restricted to differentiable objectives \citep{DBLP:journals/corr/ChangUTT16, DBLP:journals/corr/DegraveHDW16}. Many critical problems and applications, however, involve non-differentiable landscapes or stochastic simulators with intractable likelihoods. Examples include Monte-Carlo simulations in particle scattering \citep{lgso56} molecular dynamics \citep{lgso1} and many more. In these scenarios, Bayesian Optimisation (BO) \citep{bbo20} has emerged as a powerful framework for sample-efficient global optimisation.

A particularly challenging extension is \emph{contextual} BO, where the objective function depends on a controllable design variable $x \in \mathcal{X}$ and an uncontrollable environmental contextual variable $c \in \mathcal{C}$ drawn from an unknown distribution $P$. While the problem simplifies to standard stochastic optimisation if $P$ is known, practical settings require $P$ to be estimated from empirical data \cite{Boldyrev2025AnAT, Ramazyan_2024}. This estimation leads to a \emph{distributional mismatch} between the true and the nominal context distributions, potentially resulting in sub-optimal designs.

The \emph{Distributionally Robust Optimisation} (DRO) paradigm safeguards against such errors \citep{staib2019distributionally, kuhn2024wassersteindistributionallyrobustoptimization}, and has recently been integrated with BO as \emph{Distributionally Robust Bayesian Optimisation} (DRBO) \citep{kirschner2020distributionally, husain2023distributionallyrobustbayesianoptimization, micheli2025wassersteindistributionallyrobustbayesian}. DRO uses statistical divergences or distance metrics to construct an \emph{ambiguity set} - a family of candidate distributions centred around the nominal one. By optimising for the worst-case scenario within this set, DRO provides a robust performance lower bound, ensuring reliability even under distributional shifts.

\textbf{Contributions.} Our primary contributions are as follows:
\begin{itemize}[leftmargin=*]
    \item We propose \emph{Ensemble Distributionally Robust Bayesian Optimisation} (EDRBO), a novel framework that integrates an ensemble of Gaussian process surrogates with DRO to hedge against both black-box function and contextual distributional uncertainties. Our approach leverages the properties of Wasserstein DRO to derive a computationally tractable acquisition function, effectively resolving the maximin challenge inherent in DRO while preserving efficiency required for practical application.
    \item We provide a rigorous theoretical analysis, deriving a sublinear cumulative regret of $\mathcal{O}(\gamma_T \sqrt{T})$ order, where $T$ is the time horizon and $\gamma_T$ denotes the ensemble maximum information gain.
    \item We demonstrate the efficacy of EDRBO through extensive experiments, showing that it consistently matches or outperforms state-of-the-art methods while maintaining computational tractability.
\end{itemize}

The main text of this paper is organised as follows. Section~\ref{sec:related_work} presents related works. Section~\ref{sec:prelim} establishes the necessary preliminaries. In Section~\ref{sec:method}, we detail the EDRBO framework and present our theoretical results. Finally, Section~\ref{sec:experiments} provides an empirical evaluation and discussion of our findings. Section~\ref{sec:discussion} concludes the study.

\label{notation}
\paragraph{Notation.} We denote the search space by $\mathcal{X} \subseteq \mathbb{R}^{d_x}$ and the context space by $\mathcal{C} \subseteq \mathbb{R}^{d_c}$. We make a common and mild assumption that $\mathcal{X}$ is compact. Let $\Xi = \mathcal{X} \times \mathcal{C}$ and $\xi = [x^T, c^T]^T$. The black-box function $f: \Xi \rightarrow \mathbb{R}$ is assumed to be bounded from above. The ensemble surrogate model $\hat{f}$ is comprised of $M$ distinct constituent surrogates $\hat{f}_m, \ m \in \{1, \ldots, M\}, M \geq 2$. Each surrogate is defined by a positive semi-definite kernel $k_m: \Xi \times \Xi \rightarrow \mathbb{R}$ with $\mathcal{H}_m$ being the unique RKHS associated with the kernel $k_m$, equipped with norm $\|\cdot\|_{\mathcal{H}_m}$; $k_m$ is $L_m(x)$-Lipschitz with respect to $c$. We assume that $\|f\|_{\mathcal{H}_m} \leq B_m, B_m > 0$ and that $\forall \xi, \xi' \in \Xi \ k_m(\xi, \xi') \leq 1$. We use $\mathbb{P}(\mathbb{R}^{d_c})$ to denote the set of all Borel probability measures on $\mathbb{R}^{d_c}$. We use $f(x, \cdot) \#Q$ to denote the pushforward of the context distribution $Q$ through $f$. We denote the iteration budget by $T$. At iteration $t$ we use $\mathbb{D}_t = \{ (x_{\tau}, c_{\tau}, y_{\tau}) \}_{\tau=1}^T$ for the set of ground truth values and $\epsilon_t$ to denote the zero-mean $R$-sub-Gaussian output noise. We denote expected regret at iteration $t$ by $r_t$ and the cumulative expected regret is defined as $R_T = \sum_{t=1}^T r_t$.

\section{Related Works}\label{sec:related_work}

The UCB algorithm \citep{Srinivas_2012} is one of the main approaches to BO. It fits a Gaussian process to model the unknown objective function. Its posterior mean approximates for the black-box function and the posterior variance highlights areas of search space with high uncertainty. A linear combination of the two, called an acquisition function, provides a high-probability bound on the objective function. Optimising it provides viable candidates for optima. UCB acts as a base for many DRBO algorithms.

The integration of DRO into the BO framework is pioneered by \citet{kirschner2020distributionally}. DRBO-MMD works by discretising the context space and finding a solution that holds up against the worst-case distribution within a set MMD budget. While they show that the problem could be formulated for any divergence, their reliance on Maximum Mean Discrepancy creates significant scaling issues \cite{kirschner2020distributionally}. The resulting minimax problem becomes computationally prohibitive unless the context space is kept discrete and low-dimensional.

Since DRBO-MMD can be computationally heavy, DRBO-MMD Minmax \citep{tay2022efficient} offers a faster alternative. They employ Taylor expansions to conduct minimax sensitivity analysis across various ambiguity sets to achieve finer discretization without the same heavy computational overhead. Both DRBO-MMD and DRBO-MMD Minmax methods have linear regret bounds with respect to the iteration budget, limiting their long-term performance guarantees.

On the other hand, \citet{husain2023distributionallyrobustbayesianoptimization} focus on tractability. They extend the DRBO framework to $\varphi$-divergences such as Kullback-Leibler divergence and Total Variation distance. They prove that the minimax problem can be reduced to a simpler single-stage optimisation problem, thus making the algorithm much more practical. Nevertheless, their approach introduces implicit constraints regarding the support of the distributions within the ambiguity set. To handle the challenges of continuous context without relying on discretisation, SBO-KDE \citep{Huang_Song_Xue_Qian_2024} estimates the context distribution using kernel density estimation and then maximizes the expected UCB through sampling from this estimate within a $\varphi$-divergence DRBO setup. Since these density estimates can sometimes be unreliable, DRBO-KDE \citep{Huang_Song_Xue_Qian_2024} adds a layer of protection by placing a total variation ambiguity set around the estimated distribution. This hedging makes the method more resilient to potential model misspecification.

StableOpt \citep{bogunovic2018adversarially} takes a different view by looking at robustness through the lens of stability. It identifies the best design by looking for the one that performs best even under the worst plausible context in a set $C_t$. As suggested by \citet{Huang_Song_Xue_Qian_2024}, this set is defined using intervals based on the empirical mean and variance of the context dimensions. This gives us a minimax perspective that explicitly guards against statistically likely but unfavourable context shifts.

Most recently, \citet{micheli2025wassersteindistributionallyrobustbayesian} applied Wasserstein DRO to robustify the BO confidence bound acquisition function. Much like \citet{husain2023distributionallyrobustbayesianoptimization}, they bypass the maximin bottleneck by utilising the Lipschitz properties of the learner objective over Wasserstein balls. Their work effectively fills a gap in the literature concerning continuous context distributions. WDRBO \cite{micheli2025wassersteindistributionallyrobustbayesian} creates a Wasserstein ambiguity set around the observed contexts. The radius is set as $\varepsilon_t = O(1/\sqrt{t})$. It naturally tightens as more data arrives, making the model less conservative over time. \citet{micheli2025wassersteindistributionallyrobustbayesian} also propose ERBO, which is a non-robust version of WDRBO that sets the radius to zero. By simply maximising the average performance across observed contexts, ERBO serves as a context-aware non-robust baseline.

Although Gaussian process ensembles are not widely used, various examples exist. For instance \citet{candelieri2022gaussianprocessregressiondiscrete} fit $M$ Gaussian processes equipped with Squared Exponential kernels with hyperparameters varying over a fixed grid. The Gaussian processes are then combined into their Wasserstein barycentre, which reduces to the arithmetic mean of the individual means and the arithmetic mean of the individual standard deviations. In contrast, we consider Gaussian processes with kernels from distinct families to guarantee orthogonality of base experts, and hence achieve higher predictive performance of the ensemble. In addition to that, we treat the $M$ Gaussian processes as an ensemble, not just distinct runs or configuration of a surrogate model. We do so by considering a posterior radius, which measures how far each individual base expert deviates from the consensus.

\section{Preliminaries}\label{sec:prelim}

\subsection{Bayesian Optimisation} Consider optimising a black-box function $f$ with respect to the search space $\mathcal{X}$. $f$ is point-wise observable, that is, there is no direct access to neither $f$ nor its gradients. The learner may only obtain noisy evaluations of $f$. At iteration $t$, the learner selects a design point $x_t \in \mathcal{X}$. Then he observes context $c_t \in \mathcal{C}$ and output $y_t = f(x_t, c_t) + \epsilon_t$. Context follows an unknown distribution $P \in \mathbb{P}_2(\mathbb{R^{d_c}})$ and it given by the environment, that is the learner does not choose it. The goal is to find the global maximiser of $f$:
\begin{equation*}
    x^* = \argmax_{x \in \mathcal{X}} \mathbb{E}_{c \sim P} [f(x, c)].
\end{equation*}
The $(x_t, c_t, y_t)$ triplets comprise the set of all ground truth values obtained up to iteration $t$, $\mathbb{D}_t$. At iteration $t$, $f$ is modelled by a surrogate model $\hat{f}$ fit on $\mathbb{D}_t$. Using the surrogate's approximation, BO builds an \emph{acquisition function} to guide the optimisation. By optimising the acquisition function, one identifies promising regions of the search space, while accounting all that is known and unknown regarding the objective function at current iteration.

\subsection{Surrogate Model}

Consider an ensemble of $M$ Gaussian processes (GP). Following \citet{kirschner2020distributionally, micheli2025wassersteindistributionallyrobustbayesian}, each base expert solves a regularised least-squares problem in its RKHS. For a regularisation parameter $\lambda > 0$, at iteration $t$, given $\mathbb{D}_t$, the posterior of the expert with kernel $k_m$ is given by
\begin{equation} \label{eq:mean}
    \mu_{m,t}(\xi) = k_{m,t}(\xi)^T \left( K_{m,t} + \lambda I\right)^{-1} y_t
\end{equation}
and covariance
\begin{equation} \label{eq:cov}
    k_{m,t}(\xi, \xi') = \frac{1}{\lambda} \left( k_m(\xi, \xi') - k_{m,t}(\xi)^T \left( K_{m,t} + \lambda I\right)^{-1} k_{m,t}(\xi') \right),
\end{equation}
where $k_{m,t}(\xi) = [k_m(\xi, \xi_1), \ldots, k_m(\xi, \xi_t)], K_{m, t}$ is the kernel Gram matrix, and $y_t = [y_1, \ldots, y_t]$. We approximate the black-box using the experts mean:
\begin{equation}\label{eq:bbf_mean_approx}
    \hat{f}(\xi) = \mu_t(\xi) = \frac{1}{M} \sum_{m=1}^M \mu_{m, t}(\xi),
\end{equation}
for which we provide the following finite-sample confidence guarantees.
\begin{lemma}\label{lem:confidence}
    Let $\Xi \in \mathbb{R}^d$ and $f: \Xi \rightarrow \mathbb{R}$. Each expert assumes that $f \in \mathcal{H}_m$, with $\|f\|_{\mathcal{H}_m} \leq B_m$ and let $\epsilon_t$ be $R$-sub-Gaussian. Let $\delta \in (0, 1)$. With probability at least $1-\delta$, $\forall \xi \in \Xi$ and $\forall t\geq 1$:
    \begin{equation}
        |\mu_t(\xi) - f(\xi)| \leq \beta_t \sigma_t(\xi),
    \end{equation}
    with
    \begin{equation}\label{eq:std_approx}
       \sigma_t(\xi) = \frac{1}{M} \sum_{m=1}^M \sigma_{m, t} (\xi) 
    \end{equation}
    and 
    \begin{equation}\label{eq:beta_max}
        \beta_t = \max_m \lambda^{\frac{1}{2}} B_m  + R \sqrt{\gamma_{m,t} + 2 \log (M/\delta)},
    \end{equation}
    where $\gamma_{m,t}$ is maximum information gain (defined in the following subsection). A full proof is provided in Appendix \ref{sec:confidence}.
\end{lemma}
For an ensemble to be effective, the orthogonality, or distinctness, of its base experts is vital. Gaussian processes are fundamentally either equivalent or orthogonal. By the Feldman-Hájek theorem \citep{Feldman1958, Hajek1958}, distinct base experts reside in distinct RKHSs, and hence are orthogonal. We note that even under equivalence two Gaussian processes would yield varying predictions for finite data \cite{rasmussen_gp}. We equip our method with the following kernels: Square Exponential, Rational Quadratic, and Matérn kernels with $\nu=\frac{3}{2}$ and $\nu=\frac{5}{2}$. While they cover the standard suite of Lipschitz kernels, the future scalability of such ensembles pose an interesting challenge for both the theoretical and empirical aspects of the model.

\subsection{Maximum Information Gain}\label{sec:mig}

Design points are chosen to obtain the best possible estimates of the black-box function, while attempting to learn $f$ in as few iterations as possible. The reduction in uncertainty about $f$ from observing $y_t$ is called information gain and is the mutual information between $f$ and observed outputs $y_t$, i.e., $\mathcal{I}(y_t; f)$. In general, finding its maximiser is NP-hard \citep{ko1995exact}. For a Gaussian predictive posterior the problem
\begin{equation}\label{eq:mig}
    \gamma_t \defin \max_{\xi_1, \ldots, \xi_t} \  \mathcal{I}(y_t; f)
\end{equation}
is equivalent to choosing the variance-maximising design \citep{Srinivas_2012}. Thus, \emph{maximum information gain} takes the following form:
\begin{equation}\label{eq:gauss_mig}
    \gamma_{m, t} = \max_{\xi_1, \ldots, \xi_t} \log \det (I + \lambda^{-1} K_{m,t}),
\end{equation}
where $K_{m,t}$ is the kernel matrix of a Gaussian process surrogate with kernel $k_m$ over $\mathbb{D}_t$. Introduce $\gamma_t = \max_m \gamma_{m, t}$ as the maximum information gain of the ensemble.

\subsection{Distributionally Robust Optimisation}

The true context distribution $P$ is fundamentally unknown and must be inferred from finite samples. Without $P$ the decision problem cannot be evaluated. A nominal distribution $\tilde{Q}_t$ can be estimated from the observed contexts, but with finite data this estimate is subject to statistical error. To reduce this error, one must increase the sample size, which is costly or impractical within the black-box optimisation framework. This distributional mismatch between $P$ and $\tilde{Q}_t$ leads to the so-called optimiser's curse \citep{kuhn2024wassersteindistributionallyrobustoptimization}, whereby improvements fail to yield better decision due to persistent distributional error.

DRO mitigates this by employing an ambiguity set of plausible context distributions. We focus on a common choice of such ambiguity sets - Wasserstein balls. Meaning, we take the $1$-Wasserstein distance as the DRO metric, which results in the following ambiguity set centred at $Q_t$ and radius $\varepsilon_t \geq 0$:
\begin{equation}\label{eq:ball}
\mathbb{B}_{\varepsilon_t}(Q_t) \defin \{ Q \in \mathbb{P} : W(Q, Q_t) \leq \varepsilon_t \}.
\end{equation}

The radius $\varepsilon_t$ bounds the allowable distributional deviation. In the DRO framework, the learner has to make the best decision against the worst-case distribution within the ambiguity set. This results in a maximin problem:
\begin{equation}\label{eq:dro_obj}
\max_{x \in \mathcal{X}} \ \inf_{Q \in \mathbb{B}_{\varepsilon_t}(Q_t)} \ \mathbb{E}_{c \sim Q}[f(x, c)].
\end{equation}
DRO is based on the assumption that the true contextual distribution lies within the ambiguity set: $P \in \mathcal{B}_{\varepsilon_t}(Q_t)$.

\section{Method}\label{sec:method}

The key component of BO is the acquisition function. There are several established techniques, most notably the Upper Confidence Bound (UCB) acquisition. It leverages the trade off between exploitation, $\mu_t$, and exploration, $\sigma_t$. Directly optimising $\mu_t$ is a good worst-case for $f$ \citep{ego_convergence}. But it may be too greedy, has no mechanism to escape suboptimal solutions, and does not incorporate uncertainty in any form \cite{Srinivas_2012}. For UCB, by Lemma \ref{lem:confidence}, we can guarantee coverage of the black-box function $f$ with high probability. The UCB acquisition function in the DRBO paradigm is given by the following learner problem:

\begin{equation}\label{eq:ucb_acqf}
    x_t = \argmax_{x \in \mathcal{X}} \inf_{Q \in \mathcal{B}_{\varepsilon_t}(Q_t)} \mathbb{E}_{c \sim Q} [\mu_t(x, c) + \beta_t \sigma_t(x, c)].
\end{equation}

Similarly to \cite{micheli2025wassersteindistributionallyrobustbayesian} and unlike \cite{kirschner2020distributionally} or \cite{husain2023distributionallyrobustbayesianoptimization}, the algorithm and the following theoretical analysis works for any continuous distribution.

We take advantage of the chosen surrogate mode. Unlike, for example, Bayesian model averaging, which collapses distinct base experts into a single predictive density, we retain the full set of component distributions to characterise the epistemic landscape. This allows us to rigorously disentangle structural ambiguity - disagreement between kernels, from aleatoric uncertainty - lack of data within a kernel.

Consider a meaningful consensus of base experts with respect to the Wasserstein geometry - the Wasserstein Barycentre, which is the solution to the problem:
\begin{equation}
\bar{f}_t(\xi) = \argmin_{\bar{f}} \frac{1}{M} \sum_{m=1}^M W_2(\hat{f}_{m,t}(\xi), \bar{f}(\xi)).
\end{equation}

In general, finding $\bar{f}(\xi)$ is NP-hard \citep{agueh2011barycenters}. However, since the posterior of base experts is $1$D Gaussian, the ensemble barycentre is analytically tractable. It remains Gaussian, $\bar{f}_t(\xi) \sim \mathcal{N}(\bar{\mu}_t(\xi), \bar{\sigma}_t^2(\xi))$, with $\bar{\mu}_t(\xi) = M^{-1} \sum_{m=1}^M \mu_{m, t}(\xi)$ and $\bar{\sigma}_t(\xi) = M^{-1} \sum_{m=1}^M \sigma_{m, t}(\xi)$, which match $\mu_t$, Eq. \ref{eq:bbf_mean_approx}, and $\sigma_t$, Eq. \ref{eq:std_approx}, respectively. The barycentre effectively averages base experts instead of linearly mixing their densities. In contrast with the common single GP paradigm, the posterior radius naturally promotes exploration in regions where experts conflict not only in their predictions but also in their uncertainty estimates, effectively capturing the model-based epistemic uncertainty.

Thus, we can construct the robustness bound to be valid regardless of which expert defines the true context best and the acquisition function defines the exploration budget based on context gap - the distance between the assumed regularity of the most divergent expert and the consensus average. We estimate the radius $\hat{\varepsilon}_t(\xi)$ as the largest $2$-Wasserstein distance between any single constituent expert and the consensus barycentre. Leveraging the closed-form Bures-Wasserstein metric for Gaussians, this radius decomposes explicitly into structural and aleatoric components:

\begin{align}
    \mathbb{E}_{c \sim Q_t}[\hat{\varepsilon}_t(x, c)] &= \max_{m} W_2(\hat{f}_{m,t}(x, \cdot) \# Q_t, \bar{f}(x, \cdot) \# Q_t) = \\
    &= \max_{m} \mathbb{E}_{c \sim Q_t} \left[ \sqrt{\smash[b]{
        \underbrace{(\mu_{m,t}(x, c) - \mu_t(x, c))^2}_{\text{Structural Disagreement}}
      + \underbrace{(\sigma_{m,t}(x, c) - \sigma_t(x, c))^2 }_{\text{Confidence Dispersion}}
    }}\ \right]
    \vphantom{\underbrace{(\mu)^2}_{\text{Structural Disagreement}}}.
\end{align}

To simplify the DRO problem, we apply the following Wasserstein robust regulariser.
\begin{lemma}[Lemma 2 by \citet{gao2020wassersteindistributionallyrobustoptimization}]\label{lemma:lip_bound}
For an $L_g(x)$-Lipschitz function $g$, $\forall x \in \mathcal{X}$, $\forall Q \in \mathbb{B}_{\varepsilon_t}(Q_t)$ and $\forall t \geq 1$:
\begin{equation*}
    |\mathbb{E}_{c \sim Q}[g(x, c)] - \mathbb{E}_{c \sim Q_t}[g(x, c)]| \leq \varepsilon_t L_g(x).
\end{equation*}
\end{lemma}
For Wasserstein ambiguity sets, Lemma \ref{lemma:lip_bound} bounds the solution to the inner optimisation problem. This result discards the inherit minimax DRO problem and yields a tractable approximation of the learner objective. To apply Lemma \ref{lemma:lip_bound}, we first establish the Lipschitzness of the chosen acquisition function, which follows from Lemma 3 by \citet{micheli2025wassersteindistributionallyrobustbayesian}. A full proof is provided in Appendix \ref{sec:lip_constant}. Thus, the distributionally regularised learner problem is given by the following theorem.

\begin{theorem}\label{thm:acqf}
The distributionally robust learner optimisation problem is given by:
    \begin{equation}\label{eq:actual_acqf}
    x_t = \argmax_{x \in \mathcal{X}} \mathbb{E}_{c \sim Q_t}\left[\mu_t(x, c) + \beta_t \sigma_t(x, c) + \frac{1}{\sqrt{2}}\hat{\varepsilon}_t(x, c) \right].
\end{equation}
A full proof is provided in Appendix \ref{sec:post_radius_bound}.
\end{theorem}

\begin{algorithm}[t]
  \caption{General EDRBO}
  \label{alg:edrbo}
  \begin{algorithmic}
    \State {\bfseries Input:} Iteration budget $T$, $M$ GPs, initial data $\mathbb{D}_0$, trade-off parameter $\beta_t$
    \For{$t=1$ {\bfseries to} $T$}
        \For{$m=1$ {\bfseries to} $M$}
            \State $\mu_{m,t}, \sigma_{m,t} \fit \mathbb{D}_{t-1} $
        \EndFor
        \State $x_t = \argmax_{x \in \mathcal{X}} \mathbb{E}_{c \sim Q_t}\left[\mu_t(x, c) + \beta_t \sigma_t(x, c) + \frac{1}{\sqrt{2}}\hat{\varepsilon}_t(x, c) \right]$
        \State Observe $c_t \sim P$
        \State Observe $y_t = f(x_t, c_t) + \epsilon_t$
        \State $\mathbb{D}_t \leftarrow \mathbb{D}_{t-1} \bigcup \{(x_t, c_t, y_t)\}$
    \EndFor
  \end{algorithmic}
\end{algorithm}

The proposed algorithm is presented in Algorithm \ref{alg:edrbo}. To measure its performance we introduce the notion of regret. In general it is a difference between a prediction and an oracle. The choice of oracle and the definition of regret is not unique. For a consistent comparison with existing DRBO methods, following \citet{Huang_Song_Xue_Qian_2024} and \citet{micheli2025wassersteindistributionallyrobustbayesian}, we use instantaneous expected regret. It captures the sub-optimality gap between an algorithm and the solution of the true stochastic optimisation problem with full knowledge of the black-box function $f$ and the true contextual distribution $P$:
\begin{align}\label{eq:regret}
    r_t &= \mathbb{E}_{c \sim P}[f(x^*, c)] - \mathbb{E}_{c \sim P}[f(x_t, c)], \\
    x^* &= \argmax_{x \in \mathcal{X}} \mathbb{E}_{c \sim P} [f(x, c)].
\end{align}

For our acquisition function, Eq. \ref{eq:actual_acqf}, we achieve the following sublinear cumulative regret bound, which is independent of both the distributional gap, $W(P, Q_t)$, and the ambiguity set radius, $\varepsilon_t$.

\begin{theorem}\label{thm:cum_regret}
    Let $\delta \in (0, 1)$ be a probability of failure and $\exists \ \ell \geq 1$. Then with probability at least $1 - \delta$, the cumulative instantaneous expected regret of Algorithm \ref{alg:edrbo} after $T$ iterations is bounded by:
    \begin{equation}
        R_T \leq 2 (4 + \ell) \beta_T \sqrt{T\left(\gamma_T + 2\log \frac{6}{\delta}\right)},
    \end{equation}
    which results in the $\mathcal{O}(\gamma_T \sqrt{T})$ general order of cumulative regret, where $\gamma_T$ is the ensemble maximum information gain. A full proof is provided in Appendix \ref{sec:regret}.
\end{theorem}

\section{Experiments}\label{sec:experiments}

\begin{figure*}[h]
  \centering
  \includegraphics[width=\linewidth]{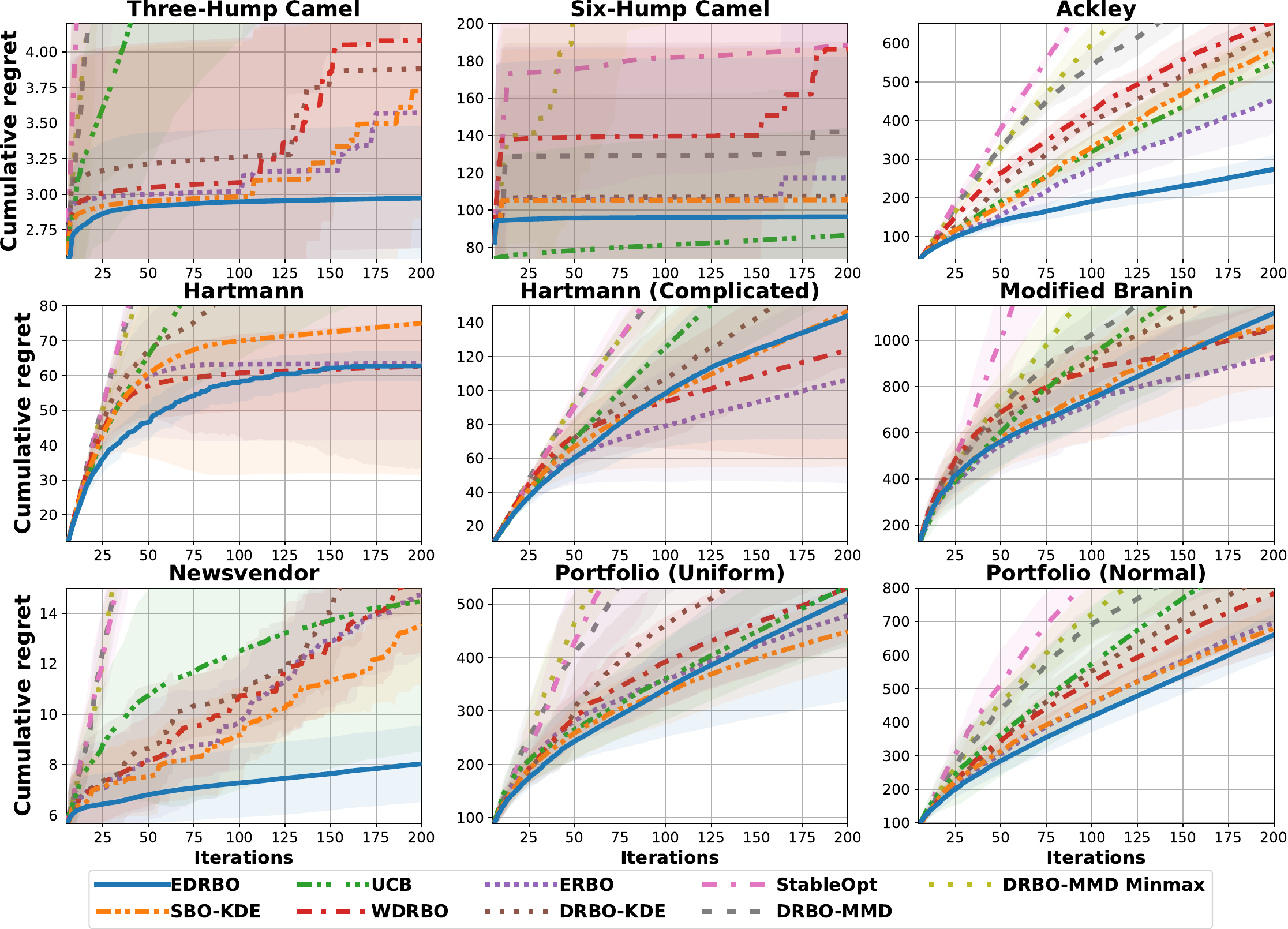}
  \caption{Mean and standard deviation of cumulative expected regret.}\label{fig:res_plots}
\end{figure*}

Our algorithm, EDRBO, follows the methodology described in Section \ref{sec:method} and the corresponding Algorithm \ref{alg:edrbo}. We compare it against a range of algorithms that take different perspectives on incorporating context and robustness in BO, ranging from purely empirical approaches to principled distributionally robust formulations.

The UCB algorithm \citep{Srinivas_2012} completely sidesteps the context variable. It relies solely on a Gaussian process model of the reward function and optimizes the corresponding UCB, making it an elegant but potentially oversimplified comparator for problems where context plays a central role.  

Several of more recent methods explicitly model the context distribution. SBO-KDE \citep{Huang_Song_Xue_Qian_2024}, estimates this distribution via kernel density estimation. The acquisition function then seeks to maximize the expected UCB under the estimated distribution, using sampling to approximate the expectation. Recognizing that KDE-based estimates may be unreliable, DRBO-KDE \citep{Huang_Song_Xue_Qian_2024} hedges against misspecification by placing a total variation ambiguity set around the KDE-derived empirical distribution. This modification makes the method more resilient to errors in the estimated context model.

A different line of work considers robustness through an MMD-based ambiguity set. DRBO-MMD \citep{kirschner2020distributionally}, discretizes the continuous context space and optimizes the UCB against the worst-case distribution supported on this discretization for a fixed MMD budget. While this formulation offers strong robustness guarantees, it suffers from a cubic computational cost in the number of discretized context points. To address this limitation, DRBO-MMD Minmax \citep{tay2022efficient}, is a scalable approximation that reduces computational burden through a minmax sensitivity analysis, thereby enabling finer discretization without prohibitive cost.  

The StableOpt algorithm \citep{bogunovic2018adversarially} takes a different view on robustness by focusing directly on stability. At each round it selects $x_t = \argmax_{x \in X} \min_{c \in C_t} UCB_t(x, c)$, where the uncertainty set $C_t$ contains the most plausible contexts given past observations. Following \citet{Huang_Song_Xue_Qian_2024}, this set is constructed by forming an interval for each context dimension centred on the empirical mean and scaled by the empirical variance. StableOpt thus embodies a minimax perspective, explicitly guarding against the worst-case realization of context within statistically plausible ranges.

Finally, the WDRBO algorithm \citep{micheli2025wassersteindistributionallyrobustbayesian}, which introduces robustness by constructing a Wasserstein ambiguity set around the empirical distribution of the observed contexts. The radius of this set shrinks over time at the rate $\varepsilon_t = O(1/\sqrt{t})$, allowing the method to adapt as more data becomes available. A natural benchmark against this approach is ERBO \citep{micheli2025wassersteindistributionallyrobustbayesian}, a non-robust variant that simply sets $\varepsilon_t = 0$. In this setting, the algorithm directly maximizes the empirical risk over the observed contexts, $x_t = \arg \max_{x \in X} \mathbb{E}_{c \sim Q_t}[UCB_t(x, c)]$, thus serving as a context-aware but non-robust baseline.

We test the algorithms on a range of synthetic and real-world problems with various context distributions. Each algorithm runs for $200$ iterations, including $5$ starting points, at $10$ different random seeds. Fig.~\ref{fig:res_plots} presents the mean and standard deviation of cumulative regret over these runs. The final values of cumulative regret are presented in Table~\ref{tab:results}. We adapt existing experimental setups \citep{Huang_Song_Xue_Qian_2024, micheli2025wassersteindistributionallyrobustbayesian}. A detailed description of all experiments can be found in Appendix \ref{sec:problems}. Additional experimental results are presented in Appendix \ref{sec:add_res} and an ablation study is presented in Appendix \ref{app:ablation}.

Empirical results provide significant insight into how EDRBO compares with existing approaches. The method achieves the best average ranking of $2.0$ across all experiments. While EDRBO notably prevails over baseline methods on Three-Hump Camel, Ackley, Newsvendor, and one of the Portfolio functions, it matches existing state-of-the-art methods on the rest. These findings are also supported by running the Wilcoxon signed-rank test. At significance $\alpha=10\%$ with Holm-Bonferroni correction, EDRBO results are significantly distinct from all baselines but SBO-KDE, WDRBO, and ERBO. Meaning, that EDRBO is on par with these methods. We report the exact test result in Appendix \ref{sec:test}.
\begin{table*}[t]
\centering
\caption{Mean $\pm$ standard deviation of final cumulative regret across all experiments. The best and the second best results per experiment are in bold and underlined, respectively. The Average rank is computed over all experiments.}\label{tab:results}
\resizebox{\textwidth}{!}{%
\begin{tabular}{@{}lccccc@{}}
\toprule
\textbf{Method} & \textbf{Three-Hump Camel}     & \textbf{Six-Hump Camel}        & \textbf{Ackley}                   & \textbf{Hartmann}             & \textbf{Complicated Hartmann} \\
\midrule
EDRBO (Ours)     & \textbf{2.97 $\pm$ 0.51}     & \underline{96.45 $\pm$ 85.24}  & \textbf{273.96 $\pm$ 62.91}       & \underline{62.78 $\pm$ 12.55} & 144.22 $\pm$ 55.01             \\
SBO-KDE          & 3.73 $\pm$ 1.09              & 105.51 $\pm$ 81.19             & 583.29 $\pm$ 160.31               & 75.04 $\pm$ 37.22             & 146.89 $\pm$ 62.65             \\
UCB              & 7.13 $\pm$ 2.03              & \textbf{86.63 $\pm$ 56.36}     & 547.97 $\pm$ 151.63               & 157.95 $\pm$ 58.07            & 222.12 $\pm$ 69.64             \\
WDRBO            & 4.08 $\pm$ 1.11              & 186.29 $\pm$ 65.43             & 651.87 $\pm$ 175.31               & \textbf{62.71 $\pm$ 14.62}    & \underline{124.04 $\pm$ 44.9}  \\
ERBO             & \underline{3.57 $\pm$ 1.06}  & 117.26 $\pm$ 83.5              & \underline{453.82 $\pm$ 127.19}   & 63.43 $\pm$ 22.43             & \textbf{106.47 $\pm$ 43.37}    \\
DRBO-KDE         & 3.88 $\pm$ 1.31              & 107.48 $\pm$ 82.88             & 628.72 $\pm$ 171.38               & 113.36 $\pm$ 56.2             & 175.69 $\pm$ 58.22             \\
StableOpt        & 56.07 $\pm$ 18.31            & 188.36 $\pm$ 65.19             & 1378.34 $\pm$ 388.01              & 369.69 $\pm$ 104.79           & 325.2 $\pm$ 92.84              \\
DRBO-MMD         & 5.99 $\pm$ 1.96              & 141.89 $\pm$ 36.26             & 809.61 $\pm$ 217.32               & 222.34 $\pm$ 62.4             & 272.38 $\pm$ 76.65             \\
DRBO-MMD Minmax  & 7.66 $\pm$ 2.99              & 287.83 $\pm$ 103.97            & 1085.92 $\pm$ 303.7               & 294.42 $\pm$ 79.69            & 275.39 $\pm$ 76.64             \\
\bottomrule
\end{tabular}}
\vspace{1em}
\resizebox{\textwidth}{!}{%
\begin{tabular}{@{}lccccc@{}}
\toprule
\textbf{Method} & \textbf{Modified Branin}          & \textbf{Newsvendor}           & \textbf{Portfolio (Uniform)}      & \textbf{Portfolio (Normal)}       & \textbf{Average rank} \\
\midrule
EDRBO (Ours)     & 1118.8 $\pm$ 267.61               & \textbf{8.03 $\pm$ 1.46}      & 509.97 $\pm$ 155.87               & \textbf{660.8 $\pm$ 157.71}      & \textbf{2.0}     \\
SBO-KDE          & 1059.25 $\pm$ 309.47              & \underline{13.54 $\pm$ 2.63}  & \textbf{448.39 $\pm$ 105.61}      & \underline{679.85 $\pm$ 160.39}  & 2.9              \\
UCB              & 1419.76 $\pm$ 460.49              & 14.49 $\pm$ 5.5               & 532.93 $\pm$ 134.73               & 951.73 $\pm$ 293.95              & 4.8              \\
WDRBO            & \underline{1050.45 $\pm$ 233.47}  & 15.37 $\pm$ 3.43              & 529.76 $\pm$ 138.79               & 784.0 $\pm$ 202.28               & 4.0              \\
ERBO             & \textbf{925.63 $\pm$ 266.35}      & 14.88 $\pm$ 3.45              & \underline{478.87 $\pm$ 105.25}   & 697.21 $\pm$ 169.05              & \underline{2.6}  \\
DRBO-KDE         & 1304.87 $\pm$ 313.51              & 17.52 $\pm$ 4.05              & 657.2 $\pm$ 179.21                & 844.75 $\pm$ 224.77              & 5.0              \\
StableOpt        & 4792.86 $\pm$ 1461.35             & 85.8 $\pm$ 24.17              & 1308.71 $\pm$ 366.79              & 1300.16 $\pm$ 354.31             & 8.9              \\
DRBO-MMD         & 1502.19 $\pm$ 371.17              & 54.94 $\pm$ 15.76             & 890.06 $\pm$ 247.67               & 981.7 $\pm$ 255.29               & 6.8              \\
DRBO-MMD Minmax  & 2153.51 $\pm$ 588.03              & 83.82 $\pm$ 24.36             & 920.08 $\pm$ 254.66               & 1110.15 $\pm$ 296.39             & 8.1              \\
\bottomrule
\end{tabular}}
\end{table*}

The main advantage of EDRBO is its handling of highly non-convex landscapes with multiple local extrema. In this case, optimisation typically takes place near vanishing gradients regions. Lipschitz-based regularised methods, such as WDRBO, rely on a linear growth assumption, which makes them less stable in these areas. In more severe scenarios, where the cost landscape exhibits localised, sharply varying structures, e.g., Ackley function, the expected cost ($\mu$) and the predictive uncertainty ($\sigma$) become strongly intertwined. In contrast to standard DRBO baselines that mostly regularise only first-order deviations, EDRBO’s use of the Bures–Wasserstein distance yields a joint second-order stability guarantee.

One may observe this behaviour in the cumulative regret plots. Most noticeably for SBO-KDE and WDRBO baselines on both Camel and Newsvendor functions. While the proposed method produces a stable and smooth cumulative regret curve, existing DRBO methods tend to produce a jagged, staircase-like result.

We also note experiments where the interaction between decision variables and context is strongly nonlinear and non-separable. A key example is the modified Branin function. Similarly, in the Portfolio optimization problems, the objective is not given analytically but instead arises as the posterior mean of a Gaussian process fitted to sampled data. This leads to nonstationary smoothness, local irregularities, and artifacts from finite data.

In conclusion, the proposed method is optimally suited for iterative decision-making under contextual uncertainty in environments where the objective function is: (i) highly non-convex, leading to frequent occurrences of stationary points, and (ii) sensitive to joint first- and second-order moment variations.

\section{Conclusions}\label{sec:discussion}

In this work, we present Ensemble Distributionally Robust Bayesian Optimisation, an ensemble surrogate algorithm for sequential decision making under context distributional uncertainty. The ensemble provides an interpretable and systematic mechanism to encode uncertainty regarding the context distribution. The method relies on robustifying the acquisition function using Lipschitz regularisation within the Wasserstein ambiguity set. In contrast with many existing distributionally robust approaches, ours yields a tractable optimisation problem in practice, while being capable of handling continuous context and having cumulative expected regret guarantees that improve on state-of-the-art methods. On a wide range of problems EDRBO attains superior to existing approaches performance, thus attesting to the theoretical results.

\section{Impact Statement}\label{sec:impact}

Ensemble Distributionally Robust Bayesian Optimization (EDRBO) contributes to foundational optimization techniques with applications across scientific and engineering design problems. Our method addresses the challenge of optimizing objectives when the operating context differs from training conditions — a common problem in expensive black-box optimization. EDRBO could be applied to molecular design, materials discovery, particle physics simulations, and aerodynamic optimization, where distributional mismatch between estimated and true context distributions introduces uncertainty. The key property of our approach — formal worst-case performance guarantees via the DRO framework — makes BO more reliable in settings where robustness to distributional shift is critical.

Improved optimization efficiency in scientific and engineering design could accelerate innovation and reduce computational costs. In drug discovery and climate science applications, this could support progress on high-impact problems. However, we note that benefits will likely accrue primarily to well-resourced institutions with access to expensive simulators, potentially reinforcing existing technological disparities. A more subtle risk involves false confidence: the worst-case guarantees provided by DRO may instill unwarranted trust if assumptions about the context distribution (e.g., choice of divergence metric) are not validated carefully for the specific application. Additionally, as with any optimization tool, EDRBO could be misused in harmful applications, though this reflects broader dual-use concerns rather than specific vulnerabilities.

To maximize benefits and mitigate risks, we encourage: (i) research on methods to help practitioners select appropriate ambiguity sets for their domains; (ii) collaboration with domain scientists to validate whether DRO assumptions hold in real-world design problems; (iii) development of diagnostic tools to assess when distributional mismatch is significant; and (iv) documentation of assumptions and limitations when EDRBO is applied to new domains. Publishing detailed case studies from scientific applications would help the community better understand real-world impact and appropriate use.

\subsubsection*{Acknowledgements}
The work was supported by the grant for research centers in the field of AI provided by the Ministry of Economic Development of the Russian Federation in accordance with the agreement 000000C313925P4E0002 and the agreement with HSE University № 139-15-2025-009. The computation for this research was performed using the computational resources of HPC facilities at HSE University \citep{Kostenetskiy_2021}.

\bibliographystyle{plainnat}
\bibliography{edrbo_ref}

%%%%%%%%%%%%%%%%%%%%%%%%%%%%%%%%%%%%%%%%%%%%%%%%%%%%%%%%%%%%%%%%%%%%%%%%%%%%%%%
%%%%%%%%%%%%%%%%%%%%%%%%%%%%%%%%%%%%%%%%%%%%%%%%%%%%%%%%%%%%%%%%%%%%%%%%%%%%%%%
% APPENDIX
%%%%%%%%%%%%%%%%%%%%%%%%%%%%%%%%%%%%%%%%%%%%%%%%%%%%%%%%%%%%%%%%%%%%%%%%%%%%%%%
%%%%%%%%%%%%%%%%%%%%%%%%%%%%%%%%%%%%%%%%%%%%%%%%%%%%%%%%%%%%%%%%%%%%%%%%%%%%%%%
\newpage
\appendix
\section*{Ensemble Distributionally Robust Bayesian Optimisation: Supplementary Materials}

\begin{itemize}
    \item[] \hyperref[app:proofs]{A. Missing Proofs} \dotfill \pageref{app:proofs}
    \item[] \begin{itemize}
        \item[] \hyperref[sec:confidence]{A.1 Regression confidence guarantees} \dotfill \pageref{sec:confidence}
        \item[] \hyperref[sec:lip_constant]{A.2 Posterior radius Lipschitzness} \dotfill \pageref{sec:lip_constant}
        \item[] \hyperref[sec:post_radius_bound]{A.3 Posterior radius bounds} \dotfill \pageref{sec:post_radius_bound}
        \item[] \hyperref[sec:regret]{A.4 Regret bounds} \dotfill \pageref{sec:regret}
    \end{itemize}
    \item[] \hyperref[app:exp]{B. Experiment details} \dotfill \pageref{app:exp}
    \item[] \begin{itemize}
        \item[] \hyperref[sec:problems]{B.1 Problems definition} \dotfill \pageref{sec:problems}
        \item[] \hyperref[sec:test]{B.2 Wilcoxon signed-rank test} \dotfill \pageref{sec:test}
        \item[] \hyperref[sec:add_res]{B.3 Additional empirical results} \dotfill \pageref{sec:add_res}
    \end{itemize}
    \item[] \hyperref[app:ablation]{C. Ablation study} \dotfill \pageref{app:ablation}
\end{itemize}

\section{Missing proofs} \label{app:proofs}

\subsection{Regression confidence guarantees}\label{sec:confidence}

To prove Lemma \ref{lem:confidence} we require the following lemma:

\begin{lemma}[Theorem 2 by \citet{DBLP:journals/corr/ChowdhuryG17}]\label{lem:single_confidence}
    Let $k: \Xi \times \Xi \rightarrow \mathbb{R}$ be the positive semidefinite kernel defining the GP surrogate, with RKHS $\mathcal{H}_k$ and norm $\|\cdot\|_{\mathcal{H}_k}$. Let $f: \Xi \rightarrow \mathbb{R}$, $f \in \mathcal{H}_k$, $\|f\|_{\mathcal{H}_k} \leq B_m$, and let $\epsilon_t$ be $R$-sub-Gaussian. Let $\delta \in (0, 1)$ be a probability of failure. Then with probability at least $1 - \delta$, $\forall \xi \in \Xi$ and $\forall t\geq 1$

    \begin{equation}
        |\mu_{m,t}(\xi) - f(\xi)| \leq \beta_{m,t} \sigma_{m,t}(\xi),
    \end{equation}

    with

    \begin{equation}\label{eq:beta_single}
        \beta_{m,t} = \lambda^{\frac{1}{2}} B_m  + R \sqrt{\gamma_{m,t} + 2 \log (1/\delta)}
    \end{equation}
\end{lemma}

\begin{lemma}[Lemma \ref{lem:confidence} in the main text]
    Let $\Xi \in \mathbb{R}^d$ and $f: \Xi \rightarrow \mathbb{R}$. Each expert assumes that $f \in \mathcal{H}_m$, with $\|f\|_{\mathcal{H}_m} \leq B_m$ and let $\epsilon_t$ be $R$-sub-Gaussian. Let $\delta \in (0, 1)$ be a probability of failure. With probability at least $1-\delta$, $\forall \xi \in \Xi$ and $\forall t\geq 1$:
    \begin{equation}
        |\mu_t(\xi) - f(\xi)| \leq \beta_t \sigma_t(\xi),
    \end{equation}
    with
    \begin{equation}
       \sigma_t(\xi) = \frac{1}{M} \sum_{m=1}^M \sigma_{m, t} (\xi) 
    \end{equation}
    and 
    \begin{equation}
        \beta_t = \max_m \lambda^{\frac{1}{2}} B_m  + R \sqrt{\gamma_{m,t} + 2 \log (M/\delta)}.
    \end{equation}
\end{lemma}

\begin{proof}

\begin{align}
    \lvert \mu_t(\xi) - f(\xi) \rvert &= \left\lvert \frac{1}{M} \sum_{m=1}^M \mu_{m,t}(\xi) - \frac{1}{M} \sum_{m=1}^M f(\xi) \right\rvert = \\
    &= \left\lvert \frac{1}{M} \sum_{m=1}^M \left( \mu_{m,t}(\xi) - f(\xi) \right) \right\rvert \leq \\
    &\markthis{(i)}{l3_1_1}{\leq} \frac{1}{M} \sum_{m=1}^M \left\lvert \mu_{m,t}(\xi) - f(\xi) \right\rvert \leq \\
    &\markthis{(ii)}{l3_1_2}{\leq} \frac{1}{M} \sum_{m=1}^M \beta_{m,t} \sigma_{m,t}(\xi) \leq \\
    &\markthis{(iii)}{l3_1_3}{\leq} \beta_t \sigma_t(\xi),
\end{align}

where \ref{l3_1_1} follows from sub-additivity of absolute value, \ref{l3_1_2} follows from applying Lemma \ref{lem:single_confidence} that holds with probability at least $1 - \frac{\delta}{M}$ for each base expert, and \ref{l3_1_3} follows from taking $\beta_t = \max_{m} \beta_{m, t}$.
\end{proof}

\subsection{Posterior radius Lipschitzness}\label{sec:lip_constant}

\begin{lemma}[Lemma 3 by \citet{micheli2025wassersteindistributionallyrobustbayesian}]\label{lem:ucb_lip}
    Let $\delta \in (0, 1)$. Then $\sigma_{m, t}$ is $\lambda^{-\frac{1}{2}} L_m(x)$-Lipschitz; with probability at least $1 - \delta$, $\mu_{m, t}$ is $\lambda^{-\frac{1}{2}} \beta_{m, t} L_m(x)$-Lipschitz. Therefore, since $\beta_{m, t} \geq 1$, with probability at least $1 - \delta$, $\mu_{m, t} + \beta_{m, t}\sigma_{m, t}$ is $2 \lambda^{-\frac{1}{2}} \beta_{m, t} L_m(x)$-Lipschitz.
\end{lemma}

Consequently, by applying this result to our surrogate, $\mu_t$ is $\max_m \lambda^{-\frac{1}{2}} \beta_{m, t} L_m(x)$ - Lipschitz with probability at least $1 - \delta$ for some probability of failure $\delta \in (0, 1)$; $\sigma_t$ is $\max_m \lambda^{-\frac{1}{2}} L_m(x)$ - Lipschitz, and $\mu_t + \beta_t \sigma_t$ is $L_t(x)$-Lipschitz with $L_t(x) = \max_m 2 \lambda^{-\frac{1}{2}} \beta_{m, t} L_{m, t}(x) = 2 \lambda^{-\frac{1}{2}} \beta_t \max_m L_{m, t}(x)$. This also gives the following property of the ensemble posterior radius:

\begin{corollary}\label{lem:lip_constant}
    Let $\delta \in (0, 1)$. Then, with probability at least $1 - \delta$, $\hat{\varepsilon}_t$ is $L_t(x)\sqrt{2}$-Lipschitz.
\end{corollary}

\begin{proof}
    Without loss of generality fix $t \geq 1$, $x \in \mathcal{X}$, and $m$, for which $\hat{\varepsilon}_t$ is computed. Put $\psi(\xi) = \mu_{m, t}(\xi) - \mu_t(\xi)$ and $\varphi(\xi) = \sigma_{m, t}(\xi) - \sigma_t(\xi)$ for some $\xi \in \Xi$. Then $\hat{\varepsilon}_t(\xi) = \sqrt{\psi^2(\xi) + \varphi^2(\xi)} = \|(\psi(\xi), \varphi(\xi))\|_2$. Consider each component $\forall \xi, \xi \in \Xi$ and apply Lemma \ref{lem:ucb_lip}:
    \begin{align}
        \left| \psi(\xi) - \psi(\xi') \right| &= \left| (\mu_{m, t}(\xi) - \mu_t(\xi)) - (\mu_{m, t}(\xi') - \mu_t(\xi'))\right| \leq \\
        &\leq \left| \mu_{m, t}(\xi) - \mu_{m, t}(\xi') \right| + \left| \mu_t(\xi) - \mu_t(\xi') \right| \leq \\
        &\leq 2 \lambda^{-\frac{1}{2}} \beta_t \max_m L_m(x)\| \xi - \xi'\| = L_t(x)\| \xi - \xi'\|,
    \end{align}
    and identically for $\varphi$
    \begin{align}
        \left| \varphi(\xi) - \varphi(\xi') \right| &= \left| (\sigma_{m, t}(\xi) - \sigma_t(\xi)) - (\sigma_{m, t}(\xi') - \sigma_t(\xi'))\right| \leq \\
        &\leq \left| \sigma_{m, t}(\xi) - \sigma_{m, t}(\xi') \right| + \left| \sigma_t(\xi) - \sigma_t(\xi') \right| \leq \\
        &\leq 2 \lambda^{-\frac{1}{2}} \max_m L_m(x)\| \xi - \xi'\| \leq L_t(x)\| \xi - \xi'\|.
    \end{align}
    Since norm is $1$-Lipschitz and $\beta_t \geq 1$, combining the two gives
    \begin{align}
        \left| \hat{\varepsilon}_t(\xi) - \hat{\varepsilon}_t(\xi') \right| &\leq \sqrt{(\psi(\xi) - \psi(\xi'))^2 + (\varphi(\xi) - \varphi(\xi'))^2} \leq \\
        & \leq \sqrt{\left( L_t(x) \|\xi - \xi'\| \right)^2 + \left( L_t(x) \|\xi - \xi'\| \right)^2} \leq \\
        & \leq L_t(x) \sqrt{2} \|\xi - \xi'\|.
    \end{align}
\end{proof}

\subsection{Posterior radius bounds}\label{sec:post_radius_bound}

\begin{theorem}[Theorem \ref{thm:acqf} in main text]\label{thm:acqf_app}
    The distributionally robust learner optimisation problem is given by:
    \begin{equation}
        x_t = \argmax_{x \in \mathcal{X}} \mathbb{E}_{c \sim Q_t}\left[\mu_t(x, c) + \beta_t \sigma_t(x, c) + \frac{1}{\sqrt{2}}\hat{\varepsilon}_t(x, c) \right].
    \end{equation}
\end{theorem}

\begin{proof}
    Consider the non-robust learner objective as the following minimax problem:
    \begin{equation}\label{eq:non_robust_af}
        x_t = \argmax_{x \in \mathcal{X}} \inf_{Q \in \mathbb{B}_{\varepsilon_t}(Q_t)} \mathbb{E}_{c \sim Q_t}\left[\mu_t(x, c) +\beta_t \sigma_t(x, c) + \frac{1}{\sqrt{2}}\hat{\varepsilon}_t(x, c) \right].
    \end{equation}
    The posterior radius, $\hat{\varepsilon}_t$, is considered ensemble-uncertainty optimistic, weighted to preserve the UCB trade-off against the design variable $x$. By Lemma \ref{lem:ucb_lip} and Corollary \ref{lem:lip_constant} the non-robust Eq. \ref{eq:non_robust_af} is $2 L_t(x)$-Lipschitz. By Lemma \ref{lemma:lip_bound} we bound the inner optimisation problem, and hence robustify against the unknown contextual distribution using:
    \begin{equation}\label{eq:robust_acqf}
        x_t = \argmax_{x \in \mathcal{X}} \ \mathbb{E}_{c \sim Q_t}\left[\mu_t(x, c) +\beta_t \sigma_t(x, c) + \frac{1}{\sqrt{2}}\hat{\varepsilon}_t(x, c) \right] - 2 \varepsilon_t L_t(x).
    \end{equation}
    Since a Gaussian process surrogate is almost always equipped with a stationary kernel, the Lipschitz constant holds for both arguments. Meaning, $L_t(x) = L_t$, which makes the robustification term constant with respect $x$. The same holds for the base experts chosen for our study. Hence the maximiser of the robustified acquisition function, Eq. \ref{eq:robust_acqf}, is equivalent to the maximiser of the learner's problem at $Q_t$, Eq. \ref{eq:actual_acqf}:
    \begin{align}
        x_t &= \argmax_{x \in \mathcal{X}} \ \mathbb{E}_{c \sim Q_t}\left[\mu_t(x, c) +\beta_t \sigma_t(x, c) + \frac{1}{\sqrt{2}}\hat{\varepsilon}_t(x, c) \right] - 2 \varepsilon_t L_t(x) = \\
        &= \argmax_{x \in \mathcal{X}} \ \mathbb{E}_{c \sim Q_t}\left[\mu_t(x, c) +\beta_t \sigma_t(x, c) + \frac{1}{\sqrt{2}}\hat{\varepsilon}_t(x, c) \right] - 2 \varepsilon_t L_t = \\
        &= \argmax_{x \in \mathcal{X}} \ \mathbb{E}_{c \sim Q_t}\left[\mu_t(x, c) +\beta_t \sigma_t(x, c) + \frac{1}{\sqrt{2}}\hat{\varepsilon}_t(x, c) \right].
    \end{align}
\end{proof}

\begin{lemma}\label{lem:radius_bound}
    Let $\delta \in (0, 1)$ be a probability of failure. Then with probability at least $1 - \delta$, $\forall t \geq 1, \forall \xi \in \Xi$ the posterior ensemble radius, $\hat{\varepsilon}_t(\xi)$, is bounded as follows
    \begin{equation}
        \hat{\varepsilon}_t(\xi) \leq 2 \sqrt{2} \beta_t \max_{m} \sigma_{m,t}(\xi).
    \end{equation}
\end{lemma}

\begin{proof}
First consider
\begin{align}
    |\mu_{m,t}(\xi) - \mu_t(\xi)| &= |\mu_{m,t}(\xi) - f(\xi) + f(\xi) - \mu_t(\xi) | \leq \\
    & \markthis{(i)}{lb4_1}{\leq} |\mu_{m,t}(\xi) - f(\xi)| + |\mu_t(\xi) - f(\xi)| \leq \\
    & \markthis{(ii)}{lb4_2}{\leq} \beta_{m,t} \sigma_{m,t}(\xi) + \beta_t \sigma_t(\xi) \leq \\
    & \markthis{(iii)}{lb4_3}{\leq} \beta_t \sigma_{m, t}(\xi) + \beta_t \sigma_t(\xi) = \\
    & = \beta_t (\sigma_{m, t}(\xi) + \sigma_t(\xi)),
\end{align}
where \ref{lb4_1} follows from sub-additivity of absolute value, \ref{lb4_2} follows by applying Lemma \ref{lem:single_confidence} and Lemma \ref{lem:confidence}, and \ref{lb4_3} follows from the definition of $\beta_t$, Eq. \ref{eq:beta_max}. Then
\begin{align}
    \hat{\varepsilon}_t(\xi) &= \max_{m} \sqrt{(\mu_{m,t}(\xi) - \mu_t(\xi))^2 + (\sigma_{m,t}(\xi) - \sigma_t(\xi))^2} \leq \\
    &\markthis{(i)}{l4_1}{\leq} \max_{m} \sqrt{\left(\beta_t (\sigma_{m,t}(\xi) + \sigma_t(\xi)) \right)^2 + (\sigma_{m,t}^2(\xi) + \sigma_t^2(\xi))} \leq \\
    &\markthis{(ii)}{l4_2}{\leq} \max_{m} \sqrt{ 2\beta_t^2 \left(\sigma_{m,t}^2(\xi) + \sigma_t^2(\xi) \right) + \beta_t^2(\sigma_{m,t}^2(\xi) + \sigma_t^2(\xi))} \leq \\
    &\markthis{(iii)}{l4_3}{\leq} \max_{m} \sqrt{ 6\beta_t^2 \sigma_{m,t}^2(\xi)} \leq \\
    &\leq 2 \sqrt{2} \beta_t  \max_{m} \sigma_{m,t}(\xi), \label{eq:dogleash}
\end{align}
where \ref{l4_1} follows from applying the bound from Eq.\ref{eq:dogleash} and the full square, \ref{l4_2} follows from $(a + b)^2 \leq 2(a^2 + b^2)$ $\forall a, b$ and $\beta_t \geq 1$, and \ref{l4_3} follows from the fact that the maximum is never less than the average.
\end{proof}

\begin{lemma}\label{lem:std_bound}
    $\forall m = 1, \ldots, M$ and $\forall \xi \in \Xi$ the posterior variance of a GP with regularisation parameter $\lambda > 0$ after $t \geq 1$ iterations is bounded as follows:
    \begin{equation}
        \sigma^2_{m, t}(\xi) \geq \frac{1}{n_t + \lambda}.
    \end{equation}
\end{lemma}
\begin{proof}
Without loss of generality fix $t \geq 1$ and $m \in 1, \ldots, M$. For a positive semidefinite kernel $k_m: \Xi \times \Xi \rightarrow \mathbb{R}$ with RKHS $(\mathcal{H}_m, \langle \cdot, \cdot \rangle_{\mathcal{H}_m} )$, let $\Phi_{m, t} = (k_m(\cdot, \xi_1), \ldots, k_m(\cdot, \xi_t))^T$ be the RKHS feature map for $\xi_1, \ldots, \xi_t \in \Xi$. Then the kernel matrix is given by $K_{m, t} = \Phi_{m, t} \Phi_{m, t}^T$ and the covariance operator $\Sigma_{m, t} = \Phi_{m, t}^T \Phi_{m, t}$.

First, we note that $\|\Phi_{m, t}\|_{\mathcal{H}_m} = \|\Phi_{m, t}^T\|_{\mathcal{H}_m}$ and $\forall a \in \mathbb{R}^t$ \ $\|\Phi_{m, t}^T a \| = \langle a, \Phi_{m, t} \Phi_{m, t}^T a \rangle_{\mathcal{H}_m} = a^T K_{m, t} a$. So $\|\Phi_{m, t}\|_{\mathcal{H}_m}^2 = \sup_{\|a\|=1} a^T K_{m, t} a = \lambda_{\mathrm{max}}(K_{m, t})$, where $\lambda_{\mathrm{max}}$ denotes the largest eigenvalue. Then $\forall g \in \mathcal{H}_m$
\begin{equation}\label{eq:norm_bound}
    \langle g, \Sigma_{m, t} g \rangle_{\mathcal{H}_m} = \|\Phi_{m, t} g \|_{\mathcal{H}_m}^2 \leq \lambda_{\mathrm{max}}(K_{m, t}) \|g\|_{\mathcal{H}_m}^2 \leq \mathrm{tr}(K_{m, t}) \|g\|_{\mathcal{H}_m}^2,
\end{equation}
where the last inequality follows from $K_{m, t} \succeq 0$. By the Cauchy-Schwartz inequality, $\forall \xi \in \Xi$
\begin{align}
    \|k_m(\cdot, \xi) \|_{\mathcal{H}_m}^2 &= \langle (\lambda I + \Sigma_{m, t}) (\lambda I + \Sigma_{m, t})^{-1} k_m(\cdot, \xi), k_m(\cdot, \xi) \rangle_{\mathcal{H}_m} \leq \\
    &\leq \langle (\lambda I + \Sigma_{m, t}) (\lambda I + \Sigma_{m, t})^{-1} k_m(\cdot, \xi), (\lambda I + \Sigma_{m, t})^{-1} k_m(\cdot, \xi) \rangle_{\mathcal{H}_m}^{\frac{1}{2}} \cdot \\
    & \hspace{45mm} \cdot \langle (\lambda I + \Sigma_{m, t}) k_m(\cdot, \xi), k_m(\cdot, \xi) \rangle_{\mathcal{H}_m}^{\frac{1}{2}} = \\
    &= \langle k_m(\cdot, \xi), (\lambda I + \Sigma_{m, t})^{-1} k_m(\cdot, \xi) \rangle_{\mathcal{H}_m}^{\frac{1}{2}} \cdot \\
    & \hspace{45mm} \langle k_m(\cdot, \xi), (\lambda I + \Sigma_{m, t}) k_m(\cdot, \xi) \rangle_{\mathcal{H}_m}^{\frac{1}{2}} = \\ 
    &= \langle k_m(\cdot, \xi), (\lambda I + \Sigma_{m, t})^{-1} k_m(\cdot, \xi) \rangle_{\mathcal{H}_m}^{\frac{1}{2}} \cdot \\
    & \hspace{45mm}  \left( \lambda \|k_m(\cdot, \xi)\|_{\mathcal{H}_m}^2 + \|\Phi_{m, t} \ k_{m}(\cdot, \xi)\|_{\mathcal{H}_m}^2 \right)^{\frac{1}{2}} \leq \\
    &\leq \langle k_m(\cdot, \xi), (\lambda I + \Sigma_{m, t})^{-1} k_m(\cdot, \xi) \rangle_{\mathcal{H}_m}^{\frac{1}{2}} \cdot \| k_{m}(\cdot, \xi)\|_{\mathcal{H}_m}(\lambda + \mathrm{tr}(K_{m, t}))^{\frac{1}{2}},
\end{align}
where the last inequality follows from applying Eq. \ref{eq:norm_bound}. The posterior variance is given by $\sigma_{m, t}^2(\xi) = \langle k_m(\cdot, \xi), (\lambda I + \Sigma_{m, t})^{-1} k_m(\cdot, \xi) \rangle_{\mathcal{H}_m}$. This, together with the kernel unit-variance bound, $\forall \xi, \xi' \in \Xi,\  k_m(\xi, \xi') \leq 1$, yields the desired bound
\begin{align}
    \sigma_{m, t}(\xi) &\geq \frac{\|k_m(\cdot, \xi)\|_{\mathcal{H}_m}^2}{\| k_{m}(\cdot, \xi)\|_{\mathcal{H}_m}\sqrt{\lambda + \mathrm{tr}(K_{m, t})}} = \\
    & = \frac{\|k_m(\cdot, \xi)\|_{\mathcal{H}_m}^2}{\sqrt{\lambda + \mathrm{tr}(K_{m, t})}} \geq \\
    & \geq \frac{1}{\sqrt{\lambda + t}}.
\end{align}
\end{proof}

\newpage
\subsection{Regret bounds}\label{sec:regret}

\begin{lemma}\label{thm:inst_exp_regret}
Let $\delta \in (0, 1)$ be a probability of failure and $\exists \ \ell \geq 1$. Then with probability at least $1 - \delta$, $\forall x \in \mathcal{X}$ and $\forall t \geq 1$ the instantaneous expected regret, Eq. \ref{eq:regret}, can be bounded as follows:
\begin{equation}
r_t \leq (4 + \ell) \beta_t \max_m \mathbb{E}_{c \sim P}[\sigma_{m, t}(x_t, c)].
\end{equation}
\end{lemma}
\begin{proof}
\begin{align}
    r_t &= \mathbb{E}_{c \sim P}[f(x^*, c)] - \mathbb{E}_{c \sim P}[f(x_t, c)] \leq \\
    &\markthis{(i)}{t2_1}{\leq} \mathbb{E}_{c \sim P}[\mu_t(x^*, c) + \beta_t \sigma_t(x^*, c)] - \mathbb{E}_{c \sim P}[\mu_t(x_t, c) - \beta_t \sigma_t(x_t, c)] \leq   \\
    &\markthis{(ii)}{t2_2}{\leq} \mathbb{E}_{c \sim Q_t}[\mu_t(x^*, c) + \beta_t \sigma_t(x^*, c)] + \varepsilon_t L_t(x^*) + \frac{1}{\sqrt{2}}\mathbb{E}_{c \sim Q_t}[\hat{\varepsilon}_t(x^*, c)] - \\
    &\hspace{44mm} - \mathbb{E}_{c \sim P}[\mu_t(x_t, c) - \beta_t \sigma_t(x_t, c)] \notag \\
    &\markthis{(iii)}{t2_3}{\leq}  \mathbb{E}_{c \sim Q_t}[\mu_t(x_t, c) + \beta_t \sigma_t(x_t, c) + \hat{\varepsilon}_t(x_t, c)/\sqrt{2} ] +  \varepsilon_t L_t(x^*) - \\
    &\hspace{44mm} -\mathbb{E}_{c \sim P}[\mu_t(x_t, c) - \beta_t \sigma_t(x_t, c)] = \notag \\
    &\markthis{(iv)}{t2_4}{\leq}  \mathbb{E}_{c \sim P}[\mu_t(x_t, c) + \beta_t \sigma_t(x_t, c) + \hat{\varepsilon}_t(x_t, c)/\sqrt{2} ] +  \varepsilon_t(L_t(x^*) + 2 L_t(x_t)) - \\
    &\hspace{44mm} -\mathbb{E}_{c \sim P}[\mu_t(x_t, c) - \beta_t \sigma_t(x_t, c)] = \notag \\
    &\markthis{(v)}{t2_5}{=}  2\beta_t \mathbb{E}_{c \sim P}[\sigma_t(x_t, c)] + \frac{1}{\sqrt{2}} \mathbb{E}_{c \sim Q_t}[\hat{\varepsilon}_t(x_t, c)] + 3 \varepsilon_t L_t \leq \\
    &\markthis{(vi)}{t2_6}{\leq} 2\beta_t \mathbb{E}_{c \sim P}[\sigma_t(x_t, c)] + 2 \beta_t \max_m \mathbb{E}_{c \sim P}[\sigma_{m, t}(x_t, c)] + 3 \varepsilon_t (2 \lambda^{-\frac{1}{2}} \beta_t \max_m L_m)  \\
    &\markthis{(vii)}{t2_7}{\leq} 4\beta_t \max_m \mathbb{E}_{c \sim P}[\sigma_{m, t}(x_t, c)] + \ell \beta_t \max_m \mathbb{E}_{c \sim P}[\sigma_{m, t}(x_t, c)] = \\
    & = (4 + \ell) \beta_t \max_m \mathbb{E}_{c \sim P}[\sigma_{m, t}(x_t, c)] , 
\end{align}
where \ref{t2_1} follows from Lemma \ref{lem:confidence}, \ref{t2_2} follows from Lemma \ref{lemma:lip_bound} and adding the weighted posterior radius as it is non-negative, \ref{t2_3} follows from $x_t$ being the maximiser of \ref{eq:actual_acqf}, \ref{t2_4} follows from adding and subtracting $2 \varepsilon_t L_t(x_t)$ and applying Lemma \ref{lemma:lip_bound}, \ref{t2_5} follows from applying the fact that for stationary kernels the Lipschitz constant holds for both arguments (see proof of Theorem \ref{thm:acqf_app}), \ref{t2_6} follows from applying Lemma \ref{lem:radius_bound} and multiplying by $\beta_t \geq 1$, \ref{t2_7} follows from Lemma \ref{lem:std_bound} and, following \cite{micheli2025wassersteindistributionallyrobustbayesian}, taking $\varepsilon_t = \mathcal{O}(t^{-\frac{1}{2}})$ and, following \cite{DBLP:journals/corr/ChowdhuryG17}, taking $\lambda = 1 + \frac{2}{T}$: $\exists \ \ell \geq 1, \ell \perp t: \sigma_t(\xi) \gtrsim \varepsilon_t L_t/ (\beta_t \ell)$.
\end{proof}

\begin{theorem}[Theorem \ref{thm:cum_regret} in main text]
    Let $\delta \in (0, 1)$ be a probability of failure and $\exists \ \ell \geq 1$. Then with probability at least $1 - \delta$, the cumulative instantaneous expected regret of Algorithm \ref{alg:edrbo} after $T$ iterations is bounded by:
    \begin{equation}
        R_T \leq 2 (4 + \ell) \beta_T \sqrt{T\left(\gamma_T + 2\log \frac{6}{\delta}\right)},
    \end{equation}
    which results in the $\mathcal{O}(\gamma_T \sqrt{T})$ general order of cumulative regret, where $\gamma_T$ is the ensemble maximum information gain.
\end{theorem}

\begin{proof}
\begin{align}
    R_T &= \sum_{t=1}^T r_t \leq \\
    &\markthis{(i)}{t4_1}{\leq} \sum_{t=1}^T (4 + \ell) \beta_t \max_m \mathbb{E}_{c \sim P}[\sigma_{m, t}(x_t, c)] \leq \\
    &\markthis{(ii)}{t4_2}{\leq} (4 + \ell) \beta_T \sum_{t=1}^T \max_m \mathbb{E}_{c \sim P}[\sigma_t(x_t, c)] = \\
    &\markthis{(iii)}{t4_3}{\leq} (4 + \ell) \beta_T \sqrt{T \sum_{t=1}^T \left(\max_m \mathbb{E}_{c \sim P} \left[ \sigma_{m,t}(x_t, c) \right]\right)^2} \leq \\
    &\markthis{(iv)}{t4_4}{\leq} (4 + \ell) \beta_T \sqrt{T \sum_{t=1}^T \max_m \mathbb{E}_{c \sim P} \left[ \sigma_{m,t}^2(x_t, c) \right] }  \leq \\
    &\markthis{(v)}{t4_5}{\leq} (4 + \ell) \beta_T \sqrt{T \left( 2 \sum_{t=1}^T \max_m \sigma_{m,t}^2(x_t, c_t) + 8 \log \frac{6}{\delta} \right) }  \leq \\
    &\markthis{(vi)}{t4_6}{\leq} (4 + \ell) \beta_T \sqrt{T \left( 2 \sum_{t=1}^T \max_m 2 \log\left(1 + \sigma_{m,t}^2(x_t, c_t) \right) + 8 \log \frac{6}{\delta} \right) }  \leq \\
    &\markthis{(vii)}{t4_7}{\leq} (4 + \ell) \beta_T \sqrt{T \left( 4 \sum_{t=1}^T \gamma_t + 8 \log \frac{6}{\delta} \right)} \leq \\
    &\markthis{(viii)}{t4_8}{\leq} 2 (4 + \ell) \beta_T \sqrt{T \left( \gamma_{T} + 2 \log \frac{6}{\delta} \right)}\label{eq:cmrb}
\end{align}

where \ref{t4_1} follows from Lemma \ref{thm:inst_exp_regret}, \ref{t4_2} follows from $\beta_T \geq \beta_t, \ \forall t \geq 1$, \ref{t4_3} follows from the Cauchy-Schwartz inequality, \ref{t4_4} follows from the Jensen inequality, \ref{t4_5} follows from Lemma 7 from \citep{kirschner2020distributionally} and noting the assumption that $\forall m, \forall \xi, \xi' \in \Xi, k_m(\xi, \xi') \leq 1$, \ref{t4_6} follows from $\forall z: |z| \leq a, z \leq 2 a \log (1 + z)$, \ref{t4_7} and \ref{t4_8} follow from the definition of maximum information gain and ensemble maximum information gain. To establish the general order:

\begin{align}
    R_T &\leq 2 (4 + \ell) \beta_T \sqrt{T \left( \gamma_{T} + 2 \log \frac{6}{\delta} \right)} \leq \\
    &\markthis{(ix)}{t4_12}{\leq} 2 (4 + \ell) \left( \max_m \left(B_m + \sigma_\epsilon \sqrt{ \gamma_{m,T} + 2 \log(M/\delta)} \right)\right) \sqrt{T \left( \gamma_{T} + 2 \log \frac{6}{\delta} \right)} \leq \\
    &\markthis{(x)}{t4_13}{\leq} 2 (4 + \ell) \left( B + \sigma_\epsilon \sqrt{\gamma_T + 2 \log(M/\delta)} \right) \sqrt{T \left( \gamma_{T} + 2 \log \frac{6}{\delta} \right)} = \\
    &= \mathcal{O}(\gamma_T \sqrt{T}),
\end{align}

where \ref{t4_12} follows from the definition of $\beta_T$ - Eq. \ref{eq:beta_max}, and \ref{t4_13} follows from taking $B = \max_m B_m$ and $\gamma_T = \max_m \gamma_{m, T}$.

\end{proof}

\newpage
\section{Experiment details}\label{app:exp}

Our code is available at \href{https://github.com/ramazyant/EDRBO}{\textbf{HERE}}.

All experimental results can be produced my running the 'main\_multiple\_runs.py' script using Python 3.11.13. All required packages are provided in the 'requirements.txt' file. A 'Results' directory will be created where all obtained results will be saved.

All experiments were run using two CPU cores of Intel Xeon Gold 6152 2.1 GHz and one GPU core of NVIDIA Tesla V100 32 Gb NVLink. We provide the mean and the standard deviation (over all experiments) of average computing time per each method in seconds in Table \ref{table:compt_time}.

We adapt the experimental setup of \citet{micheli2025wassersteindistributionallyrobustbayesian} (available at \href{https://github.com/frmicheli/WDRBO}{https://github.com/frmicheli/WDRBO}), who in their turn, adapted the experimental setup of \citet{Huang_Song_Xue_Qian_2024} (available at \href{https://github.com/lamda-bbo/sbokde}{https://github.com/lamda-bbo/sbokde}).

We adapt implementations UCB, SBO-KDE, DRBO-KDE, DRBO-MMD, DRBO-MMD Minmax, StableOpt, WDRBO, and ERBO. All algorithms, including our EDRBO, are based on the BOTorch Python package \cite{botorch}. Baseline methods are equipped with the original, default RBF kernel GP regression surrogate model. All acquisition functions are optimised using the BOTorch optimiser.

For the test problems, we reuse the Three-Hump Camel, Six-Hump Camel, Ackley, Hartmann, Hartmann (Complicated), Modified Branin, Newsvendor, Portfolio (Normal) and Portfolio (Uniform) functions.

\begin{table}[h]
  \caption{Mean and standard deviation (over all experiments) of average computing time in seconds.}\label{table:compt_time}
  \label{sample-table}
  \centering
  \begin{tabular}{lll}
    \toprule
    \textbf{Method}     & \textbf{Mean $\pm$ Std (sec.)} \\
    \midrule
    EDRBO (Ours)        & $843.16 \pm 330.48$ \\
    SBO-KDE             & $381.21 \pm 434.97$ \\
    UCB                 & $216.15 \pm 287.53$ \\
    WDRBO               & $416.98 \pm 452.97$ \\
    ERBO                & $401.09 \pm 291.74$ \\
    DRBO-KDE            & $3369.58 \pm 2653.04$ \\
    StableOpt           & $411.38 \pm 494.65$ \\
    DRBO-MMD            & $2026.80 \pm 951.64$ \\
    DRBO-MMD Minmax     & $828.01 \pm 702.28$ \\
    \bottomrule
  \end{tabular}
\end{table}

\subsection{Problem definition}\label{sec:problems}

\begin{itemize}
    \item \textbf{Three-Hump Camel.} For $x \in [0, 1]^2$ and $c \sim \mathcal{N}(0.5, 0.2^2)$ with both dimensions clipped to $[0, 1]$
    \begin{equation}
        f(x, c) = 2 x^2 - 1.05x^4 + \frac{x^6}{6} + xc + c^2.
    \end{equation}
    \item \textbf{Six-Hump Camel.} For $x \in [0, 1]^2$ and $c \sim \mathcal{N}(0.6, 0.2^2)$ with both dimensions clipped to $[0, 1]$
    \begin{equation}
        f(x, c) = \left(4 - 2.1x_1^2 + \frac{x_1^4}{3} \right) x_1^2 + x_1 x_2 + \left(-4 + 4 x_2^2 \right) x_2^2.
    \end{equation}
    \item \textbf{Ackley.} For $x \in [0, 1]^2$ and $c \sim \mathcal{N}(0.5, 0.2^2)$ clipped to $[0, 1]$
    \begin{align}
        f(x, c) = &a \exp\left( -b \sqrt{\frac{1}{3} \left( \tilde{x}_1^2 + \tilde{x}_2^2 +  \tilde{c}^2 \right)} \right) - \\ &- \exp\left( \sqrt{\frac{1}{3} \left( \cosh{h \tilde{x}_1} + \cosh{h \tilde{x}_2} +  \cosh{h c} \right)} \right) + a + e,
    \end{align}
    where $a = 20, b = 0.2, h = 2\pi, \tilde{x_i} = 65.536 x_i - 32.768, \tilde{c} = 65.536 c - 32.768$.
    \item \textbf{Hartmann.} For $x \in [0, 1]^5$ and $c \sim \mathcal{N}(0.5, 0.2^2)$ clipped to $[0, 1]$, $z = [x_1, \ldots, x_5, c]^T$
    \begin{equation}
        f(x, c) = \sum_{i=1}^4 \alpha_i \exp \left( \sum_{j=1}^6 A_{ij} \left( z_j - P_{ij} \right)^2 \right),
    \end{equation}
    where $\alpha = [1.0, 2.0, 3.0, 3.2]^T$, $A = \begin{bmatrix}
        10 & 3 & 17 & 3.50 & 1.7 & 8 \\
        0.05 & 10 & 17 & 0.1 & 8 & 14 \\
        3 & 3.5 & 1.7 & 10 & 17 & 8 \\
        17 & 8 & 0.05 & 10 & 0.1 & 14
    \end{bmatrix}$, and \mbox{$P = 10^{-4} \begin{bmatrix}
        1312 & 1696 & 5569 & 124 & 8283 & 5886 \\
        2329 & 4135 & 8307 & 3736 & 1004 & 9991 \\
        2348 & 1451 & 3522 & 2883 & 3047 & 6650 \\
        4047 & 8828 & 8732 & 5743 & 1091 & 381
    \end{bmatrix}$}.
    \item \textbf{Complicated Hartmann.} For $x \in [0, 1]^5$ and $c \in [0, 1]$, which is a mixture of six Gaussians $\mathcal{N}(0.1, 0.02^2), \mathcal{N}(0.3, 0.075^2), \mathcal{N}(0.4, 0.1^2), \mathcal{N}(0.5, 0.1^2), \mathcal{N}(0.7, 0.075^2), \mathcal{N}(0.8, 0.03^2)$, and two Cauchy distributions $Cauchy(0.2, 0.02), Cauchy(0.8, 0.02)$, clipped to $[0, 1]$.
    \item \textbf{Modified Branin.} For $x \in [0, 1]^2$ and $c \sim \mathcal{N}(0.5, 0.2^2)$ with both dimensions clipped to $[0, 1]$
    \begin{equation}
        f(x, c) = - \sqrt{h(15 x_1 - 5, 15 c_1) h(15 c_2 - 5, 15 x_2)},
    \end{equation}
    where
    \begin{equation}
        % h(u, v) = a(v - b)
        h(u, v) = a\left(v-bu^2+cu-r\right)^2+s(1 - t)*\cos{(u)}+s,
    \end{equation}
    and $a = 1$, $b = 5.1/(4\pi^2)$, $c=5/\pi$, $r=6$, $s=10$, $t=1/(8\pi)$.
    \item \textbf{Newsvendor.} For $x \in [0, 1]^2$
    \begin{equation}
        f(x, c) = 9\min{\{x, c\}} + \max{\{0, x - c\}} - 5x,
    \end{equation}
    where c $\sim$ Burr Type XII distribution with PDF
    \begin{equation}
    p(c; \alpha, \beta) = \alpha\beta\frac{c^{\alpha-1}}{(1+c^\alpha)^{\beta + 1}}
    \end{equation}
    with $\alpha=2$, $\beta=20$ and $c$ being clipped to $[0, 1]$ when it is out of the range.
    \item \textbf{Portfolio Optimisation.} For $x \in [0, 1]^3$ and $c \sim \text{U}(0, 1)$ (Uniform case) or $c \sim \mathcal{N}(0.5, 0.2^2)$ clipped to $[0, 1]$ (Normal case), the function is evaluated as the posterior mean of a Gaussian process fitted on 3,000 samples generated from CVXPortfolio \cite{boyd2017multiperiodtradingconvexoptimization}, which were re-scaled to $[0, 1]$ after the generation and provided by \cite{NEURIPS2020_e8f27796}.
\end{itemize}

\subsection{Wilcoxon signed-rank test}\label{sec:test}

\begin{table*}[h]
\centering
\caption{Pairwise comparison of EDRBO against each baseline on final cumulative regret over the nine experiments ($n = 9$ paired observations per comparison). Differences are computed as baseline minus EDRBO, so positive values favour EDRBO. Wilcoxon signed-rank $p$-values are Holm-Bonferroni-corrected over the eight comparisons; paired $t$-test $p$-values are uncorrected. A Friedman omnibus test across all nine methods rejects the hypothesis of equal performance ($\chi^2_F = 58.13$, $p = 1.08 \times 10^{-9}$).}\label{tab:stat_tests}
\resizebox{\textwidth}{!}{%
\begin{tabular}{@{}lrrrrrcrc@{}}
\toprule
 & & & \multicolumn{4}{c}{\textbf{Wilcoxon signed-rank}} & & \\
\cmidrule(lr){4-7}
\textbf{Baseline} & \textbf{Mean diff.} & \textbf{Median diff.} & $W$ statistic & $p$ value & $p_{\mathrm{Holm}}$ & \textbf{Significant at } $\alpha=10\%$ & \textbf{Paired $t$-test $p$ value} & \textbf{Significant at } $\alpha=10\%$ \\
\midrule
DRBO-KDE        & 108.39   & 50.58  & 0  & 0.0039 & 0.0313 & Yes & 0.0261 & Yes \\
DRBO-MMD        & 222.57   & 159.56 & 0  & 0.0039 & 0.0313 & Yes & 0.0074 & Yes \\
DRBO-MMD Minmax & 371.20   & 231.64 & 0  & 0.0039 & 0.0313 & Yes & 0.0127 & Yes \\
StableOpt       & 769.69   & 306.91 & 0  & 0.0039 & 0.0313 & Yes & 0.0795 & Yes \\
UCB             & 118.08   & 77.90  & 3  & 0.0195 & 0.0781 & Yes  & 0.0283 & Yes \\
WDRBO           & 58.95    & 7.34   & 12 & 0.2500 & 0.7500 & No  & 0.2185 & No \\
SBO-KDE         & 26.39    & 5.51   & 15 & 0.4258 & 0.8516 & No  & 0.4935 & No \\
ERBO            & $-1.87$  & 0.65   & 21 & 0.9102 & 0.9102 & No  & 0.9548 & No  \\
\bottomrule
\end{tabular}}
\end{table*}

\newpage
\subsection{Additional empirical results}\label{sec:add_res}

\begin{figure}[h]
  \centering
  \includegraphics[width=0.8\linewidth]{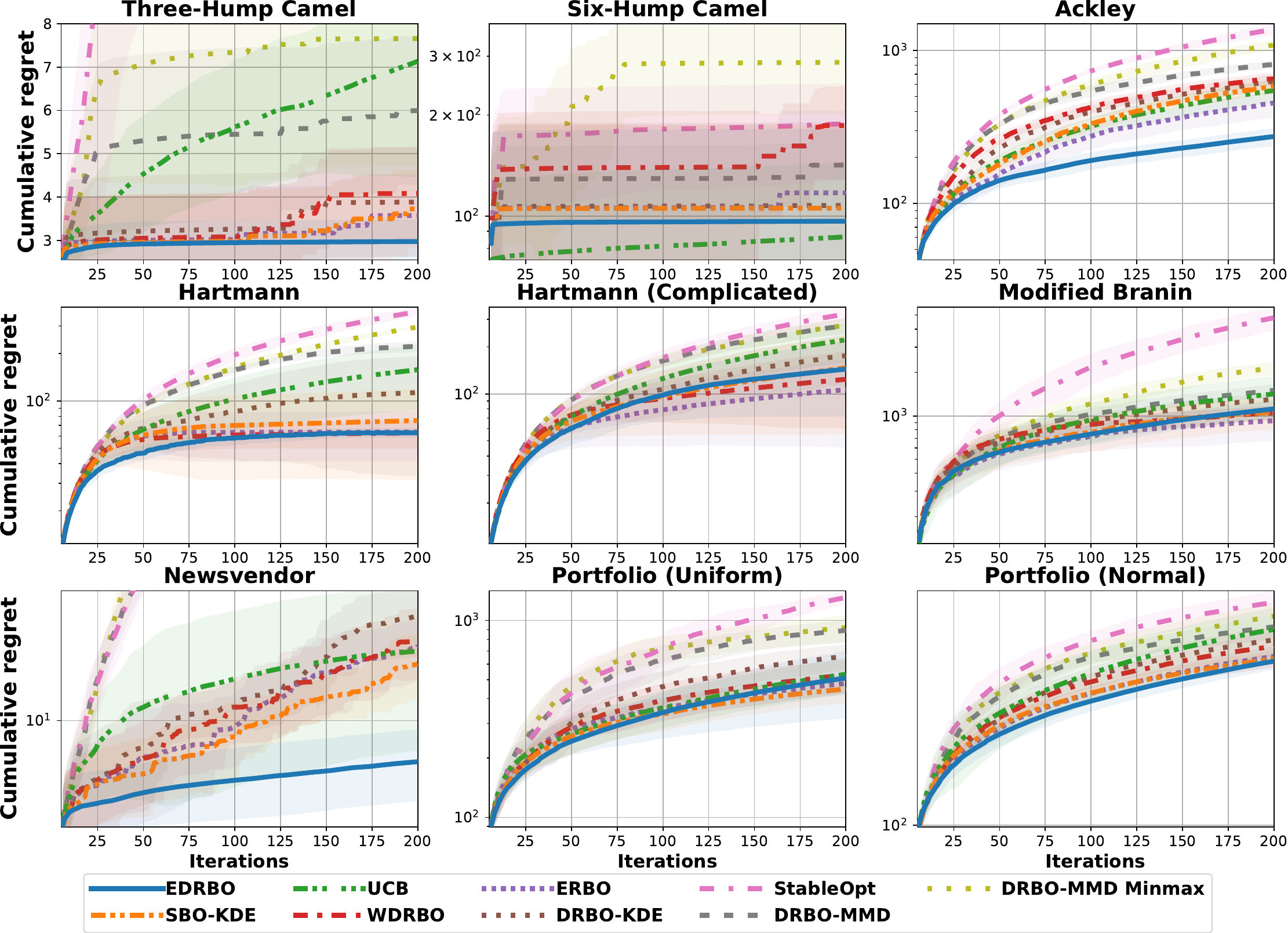}
  \caption{Mean and standard deviation of cumulative expected regret in logarithmic scale. The lower, the better.} \label{fig:inst_reg_plots}
\end{figure}

\begin{figure}[h]
  \centering
  \includegraphics[width=0.8\linewidth]{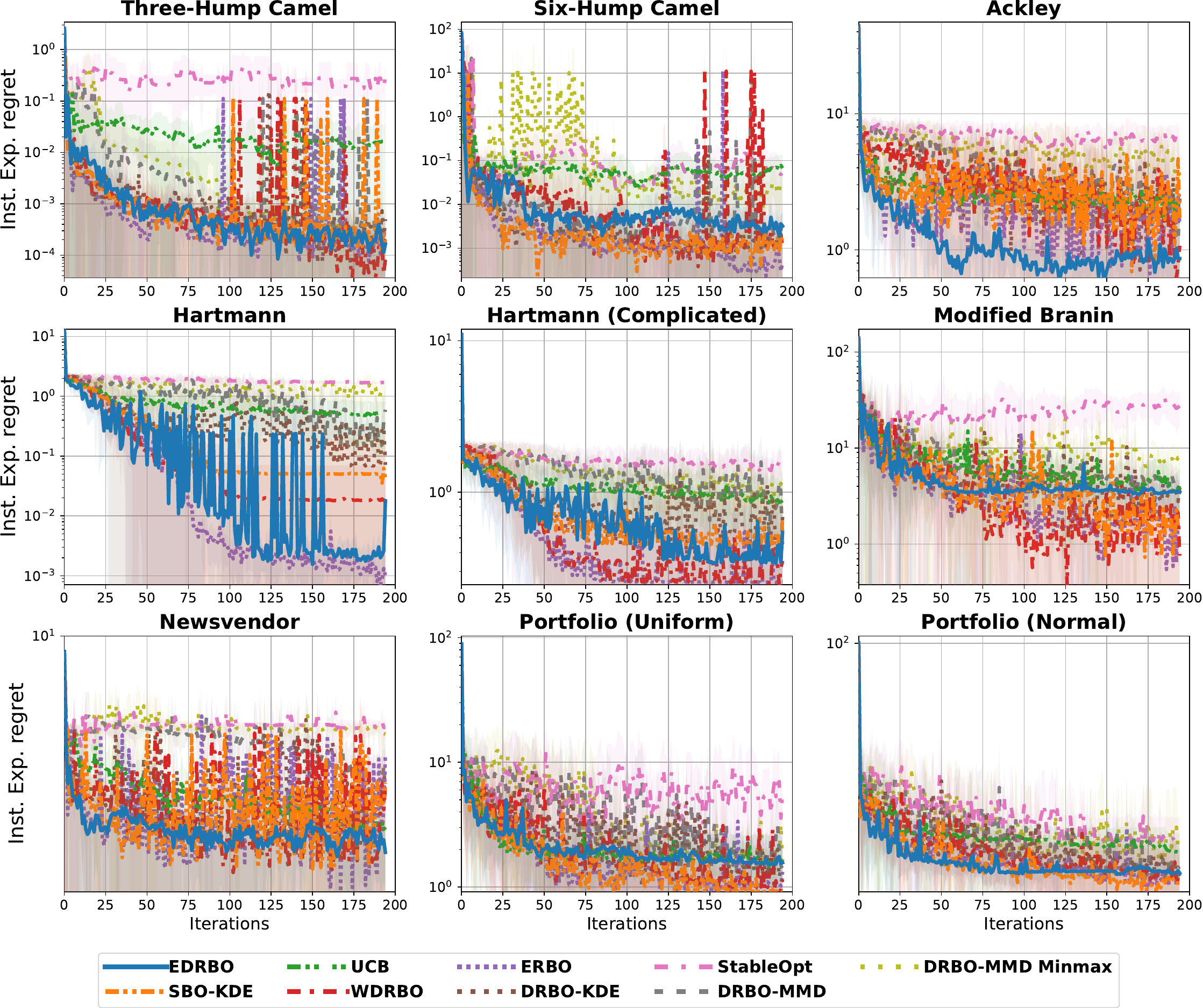}
  \caption{Mean and standard deviation of instantaneous expected regret. The lower, the better.} \label{fig:inst_reg_plots}
\end{figure}

\newpage
\section{Ablation study} \label{app:ablation}

% \begin{figure}[ht!]
%   \centering
%   \includegraphics[width=\linewidth]{Formatting_Instructions_For_NeurIPS_2026/kernels_ablation.pdf}
%   \caption{Mean and standard deviation of cumulative expected regrets.} \label{fig:ablation_cum_reg}
% \end{figure}
% \usepackage{booktabs}
% \usepackage{siunitx}

\begin{table}[ht!]
  \centering
  \caption{Paired Wilcoxon signed-rank and paired $t$-tests ($n = 9$ pairs per baseline).}
  \label{tab:ablation}
  \begin{tabular}{lccccc}
    \toprule
    Method & {Mean diff.} & {Median diff.} & {$W$} & {Wilcoxon $p_{\mathrm{Holm}}$} & {$t$-test $p_{\mathrm{Holm}}$} \\
    \midrule
    SBO-KDE                                              &   26.389 &   5.513 & 15 & 1.000 & 1.000 \\
    WDRBO                                                &   58.953 &   7.341 & 12 & 1.000 & 1.000 \\
    ERBO                                                 &   -1.871 &   0.654 & 21 & 1.000 & 1.000 \\
    Ensemble WDRBO                                       &   40.266 &  21.237 &  3 & 0.273 & 0.723 \\
    UCB                                                  &  118.082 &  77.896 &  3 & 0.273 & 0.396 \\
    EDRBO (RBF UCB)                                      &   54.412 &  19.461 &  4 & 0.328 & 0.532 \\
    EDRBO (RQ - Matérn $\frac{3}{2}$)                    &  -30.745 & -10.665 &  6 & 0.602 & 1.000 \\
    EDRBO (Matérn $\frac{5}{2}$ - Matérn $\frac{3}{2}$)  &  -15.838 &  -8.853 &  8 & 0.977 & 0.723 \\
    EDRBO (SE - RQ)                                      &   31.340 &  12.891 &  9 & 1.000 & 0.823 \\
    EDRBO (SE - RQ - Matérn $\frac{5}{2}$)               &   -0.450 &  -0.044 & 22 & 1.000 & 1.000 \\
    EDRBO (SE - Matérn $\frac{5}{2}$)                    &   16.127 &   1.528 & 15 & 1.000 & 1.000 \\
    EDRBO (RQ - Matérn $\frac{5}{2}$)                    &   -3.710 &  -8.797 & 16 & 1.000 & 1.000 \\
    EDRBO (SE - RQ - Matérn $\frac{3}{2}$)               &   -9.062 &   0.009 & 19 & 1.000 & 1.000 \\
    EDRBO (SE - Matérn $\frac{3}{2}$)                    &  -16.079 &  -0.190 & 16 & 1.000 & 1.000 \\
    \bottomrule
  \end{tabular}
\end{table}

We run an ablation on EDRBO ensemble kernels. We consider different submodels of the full EDRBO ensemble. The experimental setup is identical to the one described in the main text, see Section \ref{sec:experiments} and Appendix \ref{app:exp} for more detail. For better clarity we keep only the best performing baseline methods. In addition to that, we examine EDRBO with the UCB portion of the acquisition function being computed using RBF GP only and WDRBO with our full ensemble.

The Friedman omnibus statistic is $F = 55.658$ with $p$-value of $6.67 \times 10^{-7}$, meaning that there is difference in model performance. For a particular problem a specific combination of kernels may yield a better result than the full ensemble. However, both the Wilcoxon signed-rank and $t$-tests, presented in Table \ref{tab:ablation}, show that no comparison with full-model EDRBO is statistically significant after Holm-Bonferroni correction.

We note that in general BO methods are considered over, e.g., evolutionary algorithms, when the black-box function is computationally heavy \cite{barbetti2024lamarrlhcbultrafastsimulation}. Meaning, that fitting an ensemble of GP surrogates and employing EDRBO, should not add significant compute to the overall runtime of an experiment - we refer to the average computing time across our experiments presented in Table \ref{table:compt_time}. Without prior knowledge on the shape of the objective function, one should consider the largest feasible ensemble to maximise his safeguard and the predictive capability of the surrogate model.

\newpage
\section*{NeurIPS Paper Checklist}

\begin{enumerate}

\item {\bf Claims}
    \item[] Question: Do the main claims made in the abstract and introduction accurately reflect the paper's contributions and scope?
    \item[] Answer: \answerYes{} % Replace by \answerYes{}, \answerNo{}, or \answerNA{}.
    \item[] Justification: The abstract and introduction accurately reflect the proposed methodology, its theoretical and experimental results, which align with the content provided in Section \ref{sec:method}, Section \ref{sec:experiments} together with Appendix \ref{app:proofs} and Appendix \ref{app:exp}.
    \item[] Guidelines:
    \begin{itemize}
        \item The answer \answerNA{} means that the abstract and introduction do not include the claims made in the paper.
        \item The abstract and/or introduction should clearly state the claims made, including the contributions made in the paper and important assumptions and limitations. A \answerNo{} or \answerNA{} answer to this question will not be perceived well by the reviewers. 
        \item The claims made should match theoretical and experimental results, and reflect how much the results can be expected to generalize to other settings. 
        \item It is fine to include aspirational goals as motivation as long as it is clear that these goals are not attained by the paper. 
    \end{itemize}

\item {\bf Limitations}
    \item[] Question: Does the paper discuss the limitations of the work performed by the authors?
    \item[] Answer: \answerYes{} % Replace by \answerYes{}, \answerNo{}, or \answerNA{}.
    \item[] Justification: Section \ref{sec:experiments} discusses the obtained empirical results, which demonstrate the limitations of the proposed method.
    \item[] Guidelines:
    \begin{itemize}
        \item The answer \answerNA{} means that the paper has no limitation while the answer \answerNo{} means that the paper has limitations, but those are not discussed in the paper. 
        \item The authors are encouraged to create a separate ``Limitations'' section in their paper.
        \item The paper should point out any strong assumptions and how robust the results are to violations of these assumptions (e.g., independence assumptions, noiseless settings, model well-specification, asymptotic approximations only holding locally). The authors should reflect on how these assumptions might be violated in practice and what the implications would be.
        \item The authors should reflect on the scope of the claims made, e.g., if the approach was only tested on a few datasets or with a few runs. In general, empirical results often depend on implicit assumptions, which should be articulated.
        \item The authors should reflect on the factors that influence the performance of the approach. For example, a facial recognition algorithm may perform poorly when image resolution is low or images are taken in low lighting. Or a speech-to-text system might not be used reliably to provide closed captions for online lectures because it fails to handle technical jargon.
        \item The authors should discuss the computational efficiency of the proposed algorithms and how they scale with dataset size.
        \item If applicable, the authors should discuss possible limitations of their approach to address problems of privacy and fairness.
        \item While the authors might fear that complete honesty about limitations might be used by reviewers as grounds for rejection, a worse outcome might be that reviewers discover limitations that aren't acknowledged in the paper. The authors should use their best judgment and recognize that individual actions in favor of transparency play an important role in developing norms that preserve the integrity of the community. Reviewers will be specifically instructed to not penalize honesty concerning limitations.
    \end{itemize}

\item {\bf Theory assumptions and proofs}
    \item[] Question: For each theoretical result, does the paper provide the full set of assumptions and a complete (and correct) proof?
    \item[] Answer: \answerYes{} % Replace by \answerYes{}, \answerNo{}, or \answerNA{}.
    \item[] Justification: Notation together with the main set of assumptions is given in the corresponding paragraph in the beginning \ref{notation}. All theoretical claims, including supporting lemmas, are stated and provided with a full proof in Appendix \ref{app:proofs}, or are cited.
    \item[] Guidelines:
    \begin{itemize}
        \item The answer \answerNA{} means that the paper does not include theoretical results. 
        \item All the theorems, formulas, and proofs in the paper should be numbered and cross-referenced.
        \item All assumptions should be clearly stated or referenced in the statement of any theorems.
        \item The proofs can either appear in the main paper or the supplemental material, but if they appear in the supplemental material, the authors are encouraged to provide a short proof sketch to provide intuition. 
        \item Inversely, any informal proof provided in the core of the paper should be complemented by formal proofs provided in appendix or supplemental material.
        \item Theorems and Lemmas that the proof relies upon should be properly referenced. 
    \end{itemize}

    \item {\bf Experimental result reproducibility}
    \item[] Question: Does the paper fully disclose all the information needed to reproduce the main experimental results of the paper to the extent that it affects the main claims and/or conclusions of the paper (regardless of whether the code and data are provided or not)?
    \item[] Answer: \answerYes{} % Replace by \answerYes{}, \answerNo{}, or \answerNA{}.
    \item[] Justification: All experiments are described in Section \ref{sec:experiments} and Appendix \ref{app:exp}. Our method implementation follows the proposed methodology, Section \ref{sec:method}, and baseline methods are described and cited in Section \ref{sec:experiments} and  implementations are adapted from published research. Instructions are accompanied with the code.
    \item[] Guidelines:
    \begin{itemize}
        \item The answer \answerNA{} means that the paper does not include experiments.
        \item If the paper includes experiments, a \answerNo{} answer to this question will not be perceived well by the reviewers: Making the paper reproducible is important, regardless of whether the code and data are provided or not.
        \item If the contribution is a dataset and\slash or model, the authors should describe the steps taken to make their results reproducible or verifiable. 
        \item Depending on the contribution, reproducibility can be accomplished in various ways. For example, if the contribution is a novel architecture, describing the architecture fully might suffice, or if the contribution is a specific model and empirical evaluation, it may be necessary to either make it possible for others to replicate the model with the same dataset, or provide access to the model. In general. releasing code and data is often one good way to accomplish this, but reproducibility can also be provided via detailed instructions for how to replicate the results, access to a hosted model (e.g., in the case of a large language model), releasing of a model checkpoint, or other means that are appropriate to the research performed.
        \item While NeurIPS does not require releasing code, the conference does require all submissions to provide some reasonable avenue for reproducibility, which may depend on the nature of the contribution. For example
        \begin{enumerate}
            \item If the contribution is primarily a new algorithm, the paper should make it clear how to reproduce that algorithm.
            \item If the contribution is primarily a new model architecture, the paper should describe the architecture clearly and fully.
            \item If the contribution is a new model (e.g., a large language model), then there should either be a way to access this model for reproducing the results or a way to reproduce the model (e.g., with an open-source dataset or instructions for how to construct the dataset).
            \item We recognize that reproducibility may be tricky in some cases, in which case authors are welcome to describe the particular way they provide for reproducibility. In the case of closed-source models, it may be that access to the model is limited in some way (e.g., to registered users), but it should be possible for other researchers to have some path to reproducing or verifying the results.
        \end{enumerate}
    \end{itemize}

\item {\bf Open access to data and code}
    \item[] Question: Does the paper provide open access to the data and code, with sufficient instructions to faithfully reproduce the main experimental results, as described in supplemental material?
    \item[] Answer: \answerYes{} % Replace by \answerYes{}, \answerNo{}, or \answerNA{}.
    \item[] Justification: Code to reproduce the results in full is provided in supplementary materials. Instructions are accompanied with the code.
    \item[] Guidelines:
    \begin{itemize}
        \item The answer \answerNA{} means that paper does not include experiments requiring code.
        \item Please see the NeurIPS code and data submission guidelines (\url{https://neurips.cc/public/guides/CodeSubmissionPolicy}) for more details.
        \item While we encourage the release of code and data, we understand that this might not be possible, so \answerNo{} is an acceptable answer. Papers cannot be rejected simply for not including code, unless this is central to the contribution (e.g., for a new open-source benchmark).
        \item The instructions should contain the exact command and environment needed to run to reproduce the results. See the NeurIPS code and data submission guidelines (\url{https://neurips.cc/public/guides/CodeSubmissionPolicy}) for more details.
        \item The authors should provide instructions on data access and preparation, including how to access the raw data, preprocessed data, intermediate data, and generated data, etc.
        \item The authors should provide scripts to reproduce all experimental results for the new proposed method and baselines. If only a subset of experiments are reproducible, they should state which ones are omitted from the script and why.
        \item At submission time, to preserve anonymity, the authors should release anonymized versions (if applicable).
        \item Providing as much information as possible in supplemental material (appended to the paper) is recommended, but including URLs to data and code is permitted.
    \end{itemize}

\item {\bf Experimental setting/details}
    \item[] Question: Does the paper specify all the training and test details (e.g., data splits, hyperparameters, how they were chosen, type of optimizer) necessary to understand the results?
    \item[] Answer: \answerYes{} % Replace by \answerYes{}, \answerNo{}, or \answerNA{}.
    \item[] Justification: All experimental details are provided in Section \ref{sec:experiments} and Appendix \ref{app:exp}.
    \item[] Guidelines:
    \begin{itemize}
        \item The answer \answerNA{} means that the paper does not include experiments.
        \item The experimental setting should be presented in the core of the paper to a level of detail that is necessary to appreciate the results and make sense of them.
        \item The full details can be provided either with the code, in appendix, or as supplemental material.
    \end{itemize}

\item {\bf Experiment statistical significance}
    \item[] Question: Does the paper report error bars suitably and correctly defined or other appropriate information about the statistical significance of the experiments?
    \item[] Answer: \answerYes{} % Replace by \answerYes{}, \answerNo{}, or \answerNA{}.
    \item[] Justification: Standard deviations (over $10$ random seeds per experiment) for all results are reported. Namely, in the cumulative regret plots, Fig. \ref{fig:res_plots}, instantaneous expected regret plots, Fig. \ref{fig:inst_reg_plots}, and final cumulative regret table, Table \ref{table:final_regret}, and average computing time, Table \ref{table:compt_time}.
    \item[] Guidelines:
    \begin{itemize}
        \item The answer \answerNA{} means that the paper does not include experiments.
        \item The authors should answer \answerYes{} if the results are accompanied by error bars, confidence intervals, or statistical significance tests, at least for the experiments that support the main claims of the paper.
        \item The factors of variability that the error bars are capturing should be clearly stated (for example, train/test split, initialization, random drawing of some parameter, or overall run with given experimental conditions).
        \item The method for calculating the error bars should be explained (closed form formula, call to a library function, bootstrap, etc.)
        \item The assumptions made should be given (e.g., Normally distributed errors).
        \item It should be clear whether the error bar is the standard deviation or the standard error of the mean.
        \item It is OK to report 1-sigma error bars, but one should state it. The authors should preferably report a 2-sigma error bar than state that they have a 96\% CI, if the hypothesis of Normality of errors is not verified.
        \item For asymmetric distributions, the authors should be careful not to show in tables or figures symmetric error bars that would yield results that are out of range (e.g., negative error rates).
        \item If error bars are reported in tables or plots, the authors should explain in the text how they were calculated and reference the corresponding figures or tables in the text.
    \end{itemize}

\item {\bf Experiments compute resources}
    \item[] Question: For each experiment, does the paper provide sufficient information on the computer resources (type of compute workers, memory, time of execution) needed to reproduce the experiments?
    \item[] Answer: \answerYes{} % Replace by \answerYes{}, \answerNo{}, or \answerNA{}.
    \item[] Justification: Computer resources used in our experiments are reported in Appendix \ref{app:exp}, including mean and standard deviation of each method, Table \ref{table:compt_time}, allowing for estimation of computational cost.
    \item[] Guidelines:
    \begin{itemize}
        \item The answer \answerNA{} means that the paper does not include experiments.
        \item The paper should indicate the type of compute workers CPU or GPU, internal cluster, or cloud provider, including relevant memory and storage.
        \item The paper should provide the amount of compute required for each of the individual experimental runs as well as estimate the total compute. 
        \item The paper should disclose whether the full research project required more compute than the experiments reported in the paper (e.g., preliminary or failed experiments that didn't make it into the paper). 
    \end{itemize}
    
\item {\bf Code of ethics}
    \item[] Question: Does the research conducted in the paper conform, in every respect, with the NeurIPS Code of Ethics \url{https://neurips.cc/public/EthicsGuidelines}?
    \item[] Answer: \answerYes{} % Replace by \answerYes{}, \answerNo{}, or \answerNA{}.
    \item[] Justification: The research involves algorithmic development and evaluation on standard test functions for optimisation. It does not involve human subjects or obvious ethically sensitive applications, thus leading us to believe that it conforms to the NeurIPS Code of Ethics.
    \item[] Guidelines:
    \begin{itemize}
        \item The answer \answerNA{} means that the authors have not reviewed the NeurIPS Code of Ethics.
        \item If the authors answer \answerNo, they should explain the special circumstances that require a deviation from the Code of Ethics.
        \item The authors should make sure to preserve anonymity (e.g., if there is a special consideration due to laws or regulations in their jurisdiction).
    \end{itemize}

\item {\bf Broader impacts}
    \item[] Question: Does the paper discuss both potential positive societal impacts and negative societal impacts of the work performed?
    \item[] Answer: \answerYes{} % Replace by \answerYes{}, \answerNo{}, or \answerNA{}.
    \item[] Justification: We refer to Section \ref{sec:impact}.
    \item[] Guidelines:
    \begin{itemize}
        \item The answer \answerNA{} means that there is no societal impact of the work performed.
        \item If the authors answer \answerNA{} or \answerNo, they should explain why their work has no societal impact or why the paper does not address societal impact.
        \item Examples of negative societal impacts include potential malicious or unintended uses (e.g., disinformation, generating fake profiles, surveillance), fairness considerations (e.g., deployment of technologies that could make decisions that unfairly impact specific groups), privacy considerations, and security considerations.
        \item The conference expects that many papers will be foundational research and not tied to particular applications, let alone deployments. However, if there is a direct path to any negative applications, the authors should point it out. For example, it is legitimate to point out that an improvement in the quality of generative models could be used to generate Deepfakes for disinformation. On the other hand, it is not needed to point out that a generic algorithm for optimizing neural networks could enable people to train models that generate Deepfakes faster.
        \item The authors should consider possible harms that could arise when the technology is being used as intended and functioning correctly, harms that could arise when the technology is being used as intended but gives incorrect results, and harms following from (intentional or unintentional) misuse of the technology.
        \item If there are negative societal impacts, the authors could also discuss possible mitigation strategies (e.g., gated release of models, providing defenses in addition to attacks, mechanisms for monitoring misuse, mechanisms to monitor how a system learns from feedback over time, improving the efficiency and accessibility of ML).
    \end{itemize}
    
\item {\bf Safeguards}
    \item[] Question: Does the paper describe safeguards that have been put in place for responsible release of data or models that have a high risk for misuse (e.g., pre-trained language models, image generators, or scraped datasets)?
    \item[] Answer: \answerNA{} % Replace by \answerYes{}, \answerNo{}, or \answerNA{}.
    \item[] Justification: The paper proposes a new optimisation method. Neither the algorithm nor the standard test functions pose a high risk for misuse so as to require any safeguards. %We discuss potential misuse in Section \ref{sec:impact}.
    \item[] Guidelines:
    \begin{itemize}
        \item The answer \answerNA{} means that the paper poses no such risks.
        \item Released models that have a high risk for misuse or dual-use should be released with necessary safeguards to allow for controlled use of the model, for example by requiring that users adhere to usage guidelines or restrictions to access the model or implementing safety filters. 
        \item Datasets that have been scraped from the Internet could pose safety risks. The authors should describe how they avoided releasing unsafe images.
        \item We recognize that providing effective safeguards is challenging, and many papers do not require this, but we encourage authors to take this into account and make a best faith effort.
    \end{itemize}

\item {\bf Licenses for existing assets}
    \item[] Question: Are the creators or original owners of assets (e.g., code, data, models), used in the paper, properly credited and are the license and terms of use explicitly mentioned and properly respected?
    \item[] Answer: \answerNo{} % Replace by \answerYes{}, \answerNo{}, or \answerNA{}.
    \item[] Justification: We adapt publicly available on GitHub existing implementations of both benchmarks and baselines from published research. This is stated in Appendix \ref{app:exp}, provided with links to original implementations and corresponding publications. However, specific licences or terms of use for these assets are not explicitly mentioned in either the main text or appendix.
    \item[] Guidelines:
    \begin{itemize}
        \item The answer \answerNA{} means that the paper does not use existing assets.
        \item The authors should cite the original paper that produced the code package or dataset.
        \item The authors should state which version of the asset is used and, if possible, include a URL.
        \item The name of the license (e.g., CC-BY 4.0) should be included for each asset.
        \item For scraped data from a particular source (e.g., website), the copyright and terms of service of that source should be provided.
        \item If assets are released, the license, copyright information, and terms of use in the package should be provided. For popular datasets, \url{paperswithcode.com/datasets} has curated licenses for some datasets. Their licensing guide can help determine the license of a dataset.
        \item For existing datasets that are re-packaged, both the original license and the license of the derived asset (if it has changed) should be provided.
        \item If this information is not available online, the authors are encouraged to reach out to the asset's creators.
    \end{itemize}

\item {\bf New assets}
    \item[] Question: Are new assets introduced in the paper well documented and is the documentation provided alongside the assets?
    \item[] Answer: \answerYes{} % Replace by \answerYes{}, \answerNo{}, or \answerNA{}.
    \item[] Justification: The primary new asset is the source code for the proposed method, which is provided in the supplemental materials. Instructions are accompanied with the code.
    \item[] Guidelines:
    \begin{itemize}
        \item The answer \answerNA{} means that the paper does not release new assets.
        \item Researchers should communicate the details of the dataset\slash code\slash model as part of their submissions via structured templates. This includes details about training, license, limitations, etc. 
        \item The paper should discuss whether and how consent was obtained from people whose asset is used.
        \item At submission time, remember to anonymize your assets (if applicable). You can either create an anonymized URL or include an anonymized zip file.
    \end{itemize}

\item {\bf Crowdsourcing and research with human subjects}
    \item[] Question: For crowdsourcing experiments and research with human subjects, does the paper include the full text of instructions given to participants and screenshots, if applicable, as well as details about compensation (if any)? 
    \item[] Answer: \answerNA{} % Replace by \answerYes{}, \answerNo{}, or \answerNA{}.
    \item[] Justification: The paper does not involve crowdsourcing nor research with human subjects.
    \item[] Guidelines:
    \begin{itemize}
        \item The answer \answerNA{} means that the paper does not involve crowdsourcing nor research with human subjects.
        \item Including this information in the supplemental material is fine, but if the main contribution of the paper involves human subjects, then as much detail as possible should be included in the main paper. 
        \item According to the NeurIPS Code of Ethics, workers involved in data collection, curation, or other labor should be paid at least the minimum wage in the country of the data collector. 
    \end{itemize}

\item {\bf Institutional review board (IRB) approvals or equivalent for research with human subjects}
    \item[] Question: Does the paper describe potential risks incurred by study participants, whether such risks were disclosed to the subjects, and whether Institutional Review Board (IRB) approvals (or an equivalent approval/review based on the requirements of your country or institution) were obtained?
    \item[] Answer: \answerNA{} % Replace by \answerYes{}, \answerNo{}, or \answerNA{}.
    \item[] Justification: The paper does not involve crowdsourcing nor research with human subjects.
    \item[] Guidelines:
    \begin{itemize}
        \item The answer \answerNA{} means that the paper does not involve crowdsourcing nor research with human subjects.
        \item Depending on the country in which research is conducted, IRB approval (or equivalent) may be required for any human subjects research. If you obtained IRB approval, you should clearly state this in the paper. 
        \item We recognize that the procedures for this may vary significantly between institutions and locations, and we expect authors to adhere to the NeurIPS Code of Ethics and the guidelines for their institution. 
        \item For initial submissions, do not include any information that would break anonymity (if applicable), such as the institution conducting the review.
    \end{itemize}

\item {\bf Declaration of LLM usage}
    \item[] Question: Does the paper describe the usage of LLMs if it is an important, original, or non-standard component of the core methods in this research? Note that if the LLM is used only for writing, editing, or formatting purposes and does \emph{not} impact the core methodology, scientific rigor, or originality of the research, declaration is not required.
    %this research? 
    \item[] Answer: \answerNA{} % Replace by \answerYes{}, \answerNo{}, or \answerNA{}.
    \item[] Justification: The core development does not involve LLMs as any important, original, or non-standard component.
    \item[] Guidelines:
    \begin{itemize}
        \item The answer \answerNA{} means that the core method development in this research does not involve LLMs as any important, original, or non-standard components.
        \item Please refer to our LLM policy in the NeurIPS handbook for what should or should not be described.
    \end{itemize}

\end{enumerate}

\end{document}